\documentclass{article}

\usepackage{microtype}
\usepackage{graphicx}
\usepackage{subfigure}
\usepackage{booktabs} 
\usepackage{multicol}

\usepackage{hyperref}

\usepackage{xcolor}
\definecolor{redorange}{RGB}{255, 68, 51}
\usepackage{color-edits}
\addauthor[Ruohan]{rz}{redorange}

\usepackage[accepted]{icml2024}

\usepackage[utf8]{inputenc} 
\usepackage[T1]{fontenc}    
\usepackage{hyperref}       
\usepackage{url}            
\usepackage{booktabs}       
\usepackage{amsfonts}       
\usepackage{nicefrac}       
\usepackage{microtype}      
\usepackage{xcolor}         
\usepackage{subfigure}      
\usepackage{bbm}

\usepackage{float}

\usepackage{amsmath}
\usepackage{amssymb}
\usepackage{mathtools}
\usepackage{amsthm}
\usepackage{ying}

\newcommand{\bfR}{\mathbf{R}}
\newcommand{\bfP}{\mathbf{P}}

\newcommand{\calA}{\mathcal{A}}

\newcommand{\calE}{\mathcal{E}}
\newcommand{\bbE}{\mathbb{E}}
\newcommand{\bbP}{\mathbb{P}}
\newcommand{\calB}{\mathcal{B}}
\newcommand{\bbR}{\mathbb{R}}

\def \Zki {z_k^{(i)}}

\def \Zkone {z_k^{(1)}}
\def\oper{\mathop{\text{op}}}

\newcommand{\muhatupp}[2]{{\hat\mu_{{#1},{#2}}^U}}
\newcommand{\muhatlow}[2]{{\hat\mu_{{#1},{#2}}^L}}
\newcommand{\mutrue}[2]{{\mu_{{#1},{#2}}}}
\newcommand{\rupp}[1]{{\hat{U}_{{#1}}}}
\newcommand{\rlow}[1]{{\hat{L}_{{#1}}}}
\newcommand{\rtrue}{{r}}
\newcommand{\covmat}[2]{{V_{{#2}}^{({#1})}}}

\usepackage[shortlabels]{enumitem}

\usepackage[capitalize,noabbrev]{cleveref}

\theoremstyle{plain}
\theoremstyle{definition}
\theoremstyle{remark}

\usepackage[textsize=tiny]{todonotes}

\icmltitlerunning{Adaptively Learning to Select-Rank  in Online Platforms}

\begin{document}

\twocolumn[
\icmltitle{Adaptively Learning to Select-Rank  in Online Platforms}



\icmlsetsymbol{equal}{*}

\begin{icmlauthorlist}
\icmlauthor{Jingyuan Wang}{NYU}
\icmlauthor{Perry Dong}{Berkeley,Arena}
\icmlauthor{Ying Jin}{Stanford}
\icmlauthor{Ruohan Zhan}{HKUST,SHCIRI}
\icmlauthor{Zhengyuan Zhou}{NYU,Arena}
\end{icmlauthorlist}

\icmlaffiliation{NYU}{Stern School of Business, New York University}
\icmlaffiliation{Berkeley}{EECS, UC Berkeley}
\icmlaffiliation{Stanford}{Department of Statistics, Stanford University}
\icmlaffiliation{HKUST}{IEDA, Hong Kong University of Science and Technology (HKUST)}
\icmlaffiliation{SHCIRI}{HKUST Shenzhen-Hong Kong Collaborative Innovation Research Institute}
\icmlaffiliation{Arena}{Arena Technologies}

\icmlcorrespondingauthor{Ruohan Zhan}{rhzhan@ust.hk}
\icmlcorrespondingauthor{Zhengyuan Zhou}{zhengyuanzhou24@gmail.com}

\icmlkeywords{learning to rank, upper confidence bound, nonparametric learning algorithm, regret minimization, bandit learning, reinforcement learning}

\vskip 0.3in
]



\printAffiliationsAndNotice{}  

\begin{abstract}
Ranking algorithms are fundamental to various online platforms across e-commerce sites to content streaming services. Our research addresses the challenge of adaptively ranking items from a candidate pool for heterogeneous users, a key component in personalizing user experience. We develop a user response model that considers diverse user preferences and the varying effects of item positions, aiming to optimize overall user satisfaction with the ranked list. 
We frame this problem within a contextual bandits framework, with each ranked list as an action. Our approach incorporates an upper confidence bound to adjust predicted user satisfaction scores and selects the ranking action that maximizes these adjusted scores, efficiently solved via maximum weight imperfect matching.
We demonstrate that our algorithm achieves a cumulative regret bound of $O(d\sqrt{NKT})$ for ranking $K$ out of $N$ items in a $d$-dimensional context space over $T$ rounds, under the assumption that user responses follow a generalized linear model. This regret alleviates dependence on the ambient action space, whose cardinality grows exponentially with $N$ and $K$ (thus rendering direct application of existing adaptive learning algorithms -- such as UCB or Thompson sampling -- infeasible). Experiments conducted on both simulated and real-world datasets demonstrate our algorithm outperforms the baseline.

\end{abstract}

\section{Introduction}
\label{sec:intro}

Online platforms have significantly influenced various aspects of daily life. Ranking algorithms, central to these platforms, are designed to organize vast quantities of information to enhance user satisfaction. This has shown to be valuable for businesses: for example, YouTube uses these algorithms to present the most relevant videos for an optimal user experience, while Amazon employs them to display products that are likely to maximize revenue. 
Arena, a leading AI-driven B2B startup, uses active learning (combined with foundation models) to rank promotions, products, sales tasks (for in-store sales representatives) in omnichannel commerce for global enterprises in the consumer packaged goods industry.
This paper focuses on optimizing ranking algorithms within such  platforms. The process is twofold: (i) the retrieval/select phase, where the most relevant $K$ items are selected from a large pool, and (ii) the ranking phase, where these items are arranged in a way that aims to maximize overall user satisfaction over the entire ranked list \cite{guo2019pal,lin2021mitigating,lerman2014leveraging}.

%

Large-scale ranking algorithms employed by major companies often utilize an ``explore-then-commit'' (ETC) strategy. This approach ensures stable performance in production environments but relies heavily on passive learning, where outcomes are largely dependent on previously collected data.
In typical ETC methods, models are initially trained using historical production data. During deployment, these models rank a subset of items, aiming to achieve the highest possible user satisfaction based on predictions made by the trained model. One common such method is score-based ranking, where models assign scores to user-item pairs, predicting the level of user satisfaction with each item. Consequently, items are sorted in descending order of these scores~\cite{liu2009learning,joachims2002optimizing,herbrich1999support,freund2003efficient,burges2005learning,cao2007learning,lee2014large,li2007mcrank,li2006ordinal,burges2010ranknet}.
Alternatively, some methods employ offline reinforcement learning, where the objective  is to directly generate a ranked list of items to optimize total user satisfaction across the entire list \cite{bello2018seq2slate,wang2019sequential}.

However, a fundamental limitation of these ETC methods is the inherent estimation uncertainty: the models cannot precisely predict user responses, regardless of the volume of training samples used. This uncertainty may arise from a potential distributional shift between the training dataset and the future target audience. Such shifts are common in scenarios like introducing new items (typical in ``cold-start'' algorithms) or expanding into new markets, where the platform must extrapolate demand beyond the scope of the original training data \cite{ye2022cold,agrawal2019mnl}. 
Furthermore, even in relatively stable deployment environments, estimation uncertainty persists due to ``sparse interaction'' in the logged data \cite{chen2019top,bennett2007netflix}. While platforms might have substantial information about each item and user from past data, a considerable portion of potential user-item interactions remains unobserved and absent from the dataset. This gap, where many user-item interactions are never realized or captured, further complicates the prediction accuracy of the models.

The research community has seen significant efforts toward optimal decision-making in the face of estimation uncertainty, a key theme in bandit literature \cite{lai1985asymptotically,russo2018tutorial,thompson1933likelihood,agrawal2012analysis,auer2002finite,chu2011contextual}.  
The core idea is to engage in active learning, allowing models to be continuously updated and \emph{adaptively} optimized with incoming data, which aligns well with the sequential user interaction typical in online platforms
\cite{hu2008collaborative,agichtein2006improving,zoghi2016click}. 
Among these bandit learning methods, a seminal approach is to make decisions \emph{optimistically} in face of uncertainty, that is, to select the action with the highest potential (based on uncertainty quantification) to be optimal. In our context, this translates to presenting item rankings that have the greatest potential for maximizing user satisfaction. As new data becomes available, models are retrained and uncertainty estimations recalibrated, thereby adaptively optimizing cumulative user satisfaction over time.

However, directly applying bandit algorithms to our ranking problem would result in an NP-hard problem. The action space, which includes all possible item rankings, grows exponentially with the number of items.
To address this, researchers in the ranking bandit literature have introduced specific structures to simplify the original ranking problem \cite{kveton2015cascading,zong2016cascading,zhong2021thompson,katariya2016dcm,li2019online,lagree2016multiple,gauthier2022unirank,lattimore2018toprank,shidani2024ranking}.
These models approach user satisfaction by separately estimating item attractiveness and homogeneous position effects over items. The resulting rankings are then ordered based on item attractiveness. Furthermore, much of this research focuses on the multi-armed bandit scenario where item attractiveness parameters are considered at the population level.

Our work extends the current ranking bandit literature, bridging the gap to real-world applications. We focus on two key aspects: (i) contextual ranking, which is fundamental to the personalization at the heart of online platforms \cite{chen2019top}, and (ii) heterogeneous item-position effects, acknowledging that different items may influence users differently depending on their position in the ranking \cite{guo2019pal,collins2018study}.
Furthermore, we address a critical aspect of optimization: how to efficiently select the most effective ranking based on our estimates. We demonstrate that the ranking problem can be effectively transformed into a bipartite matching problem. This allows us to identify solutions efficiently using established graph optimization techniques. Our contributions are as follows:
\begin{itemize}
\item We introduce a contextual ranking bandit algorithm that adaptively learns to rank items to optimize cumulative user satisfaction. This algorithm has a cumulative regret upper bounded by $O(\sqrt{NKT})$ for ranking $K$ items from $N$ candidates over a time horizon of $T$.
\item We transform the optimization problem of selecting the optimal ranking action into a bipartite maximum weight imperfect matching, which we solve efficiently  in polynomial time.
\item We provide empirical evidence of our method's effectiveness over baseline models using both simulated data and production datasets from an e-commerce platform.
\end{itemize}

\subsection{Related Works}




Learning to rank has been  studied extensively in the bandit literature across various settings. In scenarios aimed at maximizing user clicks, numerous ranking bandit algorithms have been developed based on different user click models (see \cite{chuklin2022click} for an overview of click models).  A popular approach is the cascade model and its variants. These models generally assume that users browse through a list of items in a sequential order and click on the most attractive option \cite{kveton2015cascading,zong2016cascading,zhong2021thompson,katariya2016dcm}.
However, the cascade model restricts  user interaction to a single click,  which cannot capture various applications such as maximizing revenues on an e-commerce platform, or enhancing overall user satisfaction in the content streaming platform. In contrast, the model proposed in our work is designed to accommodate a wide spectrum of user responses from clicks to purchases.

Another popular category within bandit algorithms is based in  position-based models (PBM), which decompose the user response into item attractiveness and position bias \cite{lagree2016multiple,komiyama2017position,lattimore2018toprank}. However, such PBM models, as well as cascade models, usually assume  the position effects are the same across users and items, which may not hold in practice \cite{guo2019pal}.  Moreover, under the assumption of homogeneous positional effects, determining the optimal ranking becomes straightforward: items are simply ranked in descending order of attractiveness.
Recently, \citeauthor{gauthier2022unirank} propose a unimodel bandit to solve the adaptive ranking challenge. 
However, like PBMs and cascade models, this algorithm also presupposes that the optimal ranking should be based on descending item attractiveness.
In contrast,  our approach allows for heterogeneous position effects that vary across different users and items, and we propose an efficient optimization method to identify the optimal ranking under our  model, offering  a more nuanced and practical solution for real-world applications.

Moreover, many ranking bandit studies overlook the rich contextual information available from individual-level data on online platform.
Along this line, \citeauthor{li2019online} propose ranking items with input features, but fail to accommodate the varied responses of different users, a key aspect for personalization. 
In response, we propose a novel reward model that leverages user features in each interaction to learn item parameters driving heterogeneous user responses. Our work is built upon a wide literature of contextual bandits \cite{li2016contextual,li2017provably} but adeptly adapts them to personalized ranking.

Finally, our work is also closely related to literature on combinatorial bandits and adaptive assortment problems \cite{agrawal2019mnl,qin2014contextual,chen2013combinatorial,li2016contextual,combes2015combinatorial}.
These areas typically involve selecting a subset of items from a candidate pool, akin to the initial phase of our problem. However, they often do not include the critical ordering phase that our work emphasizes.
In many practical applications, the position of content significantly influences user attention and feedback \cite{craswell2008experimental,collins2018study,zhao2019recommending}. Recognizing this, our paper aims to determine not just the optimal selection but also the optimal order of content to maximize user satisfaction across the entire ranked list.







\section{Problem Setup}

We describe the adaptive ranking problem of online platforms as follows. Consider a platform that hosts $N$ items. Our goal is to optimize the agent for joint retrieval and ordering, which determines the optimal display order of $K$ items from the total $N$ candidates for incoming customers.\footnote{Most large-scale, industry-standard recommendation systems include two steps: retrieval and ranking. The retrieval phase identifies the $K$ most relevant item candidates for a specific user from a large pool of items, and the ranking agent determines the optimal display order of these $K$ items \cite{googleRS}.} 
We focus on learning the underlying embedding of items given a specified interaction form between users and items. 
\begin{remark}
Our framework admits unstructured items, i.e., there are no item features given exogenously as context information. This framework can easily be adapted to  a wide range of ranking scenarios. Appendix~\ref{sec:extension_feature} elaborates on the application of our framework to ranking structured 
items, where item features are provided and the platform learns parameters of the interaction model between items and users.
\end{remark}



Let $T$ be the horizon of the bandit experiment. At each time step $t\in[T]$, \footnote{We use $[n]$ to denote ${1,\dots, n}$ for any $n\in\mathbb{N}^+$.} a user arrives with a context $X_t\in \bfR^{d}$, which is independently and identically distributed (i.i.d.) as population $\bfP_X$. Let $s_t(X_t)=(q_t(1), \dots, q_t(K))$ be the retrieved $K$ items for the user $X_t$, where $q_t(k)\in[N]$ denotes the $k$-th item in the  set $s_t(X_t)$. Note that $s_t(X_t)$ depends on the context $X_t$, i.e. the retrieved sets may vary for different user contexts.
 Our goal is to (i) optimally choose the retrieved $K$ items (i.e., the collection of items in the retrieval phase) and (ii) decide the display order of the retrieved items (i.e., the order of $K$ displayed items in the ranking phase). We make the following assumption on contexts and retrieved items.

\begin{assumption}[Context Variation]\label{assump:context}
There exists a constant $c_x>0$ such that  $\Sigma_{x,j}:=\EE[X_tX_t^\top\given j\in s_t(X_t)]\succeq c_1\cdot I$ for all $j\in[N]$.
These expectations are both taken over  $X_t\sim \PP_X$. For notation convenience, let the feature vector be rescaled as 
$\sigma_t(k) \leftarrow \sigma_t(j)/K-1/2\in(-1/2,1/2)$, 
and  
$\|X_t\|\leq \sqrt{3}/2$, 
such that $\|Z_{t,k}\|\leq 1$, where $Z_{t,k}=(\sigma_t(k), X_t)$. 
\end{assumption}

The context variation condition is standard in contextual bandit literature, as in \cite{li2017provably, zhou2020neural}.

Next, we describe the platform-user interaction. At any time $t$, for each item $j$ ranked in position $k$, let $Y_{t,j,k}$ be the potential outcome of the user satisfaction with this item. We also let $S_K$ be the set of all possible permutations of $K$ items. The interaction between the platform and the user is as follows. At each time $t$, the ranking agent generates a ranking $\sigma_t$ for the observed user $X_t$ and receives  user  feedback $Y_{t, q_t(k), \sigma_t(k)}$  on each $k$-th item $q_t(k)$ in the previously retrieved set $s_t(X_t)$. Such model with item-wise user response is widely adopted in many applications, for example streaming platforms (such as Netflix) have access to the user's watchtime on each recommended content. 
The ranking agent aims  to generate a sequence of rankings $\{\sigma_1, \dots, \sigma_T\}$ over a bandit experiment of horizon $T$ to minimize the \emph{cumulative regret} over the ranked lists, which is defined as below:
\begin{align}
\label{eq:cum_regret}
    \sum_{t=1}^T r(X_t, \sigma^*_t) - r(X_t, \sigma_t),
\end{align}
where $r(X_t,\sigma_t)$ is the expected user satisfaction 
under context $X_t$ and ranking $\sigma_t$ and $\sigma^*_t := \argmax_{\sigma\in S_K} r(X_t, \sigma)$ denotes the optimal ranking at time $t$, under context $X_t$.

The remaining of this section focuses on details of our user satisfaction model, the reward structure and some real-world examples.

\subsection{User Satisfaction Model} 
We assume that user interactions with individual items admit a generalized linear model, while user satisfaction across an entire ranked list adheres to a generalized additive model (which essentially aggregates user satisfaction on each item). 
This setup is quite broad and captures several important ranking applications in practice, including click-through-rate and revenue modeling, discussed in more detail in Section~\ref{sec:reward-outcome-structure}. 


Formally, at time $t$, for each item $j$ ranked in position $k$, let $Y_{t,j,k}$ be the potential outcome of  the user satisfaction with this item.
We assume that the conditional distribution of $Y_{t,j,k}$ follows a generalized linear model in an exponential family, 
\begin{align*}
&\bbP(Y_{t,j,k}|X_t; \beta_j,\alpha_j)\\
=&h(Y_{t,j,k},\tau)\exp\Big(\frac{Y_{t,j,k}(\alpha_j k + \beta_j^T X_t) - A(\alpha_j k + \beta_j^T X_t)}{d(\tau)}\Big),
\end{align*}
where $h(\cdot), d(\cdot), A(\cdot)$ are specified functions, $\tau$ is the known scale parameter, $\beta_j\in \bfR^{d}$ is the unknown embedding of item~$j$, and $\alpha_j\in\bbR$ is the unknown position effect of item~$j$. 

For the learning purpose, we are interested in estimating the item-specific parameters: the embedding $\beta_j$ and the position effect $\alpha_j$, based on which we can compute the conditional expectation of $Y_{t,j,k}$:
\begin{align}
\label{eq:glm_model}
    \mu_{j}(X_t, k) := \bbE[Y_{t,j,k}|X_t, j,k] = A'(\alpha_j k + \beta_j^T X_t),
\end{align}
where $A'(\cdot)$ is the derivative of $A(\cdot)$.

\subsection{Reward and Outcome Structure}\label{sec:reward-outcome-structure}
Given a ranking $\sigma=(\sigma(1),\dots, \sigma(K))\in S_K$, we assume the expected user satisfaction of the ranked list is additive:
\begin{align}\label{eqn:additive-reward}
r(X_t, \sigma)=\sum_{k=1}^K \mu_{q_t(k)}(X_t,\sigma(k)).
\end{align}
We present the following examples to motivate such additive reward structure.

\begin{example}[Watchtime]\label{ex:watchtime} For many streaming services, such as short video platforms (TikTok) and video streaming platforms (Netflix and YouTube), the goal is to optimize the total amount of time that users have watched on the platform instead of a single video. Let the user satisfaction outcome $Y_{t,j,k}\geq0$ represents the user watchtime of the video $j$ ranked in position $k$ at time $t$. The reward structure at any time $t$ with a given ranking $\sigma$ is the sum of all $K$ retrieved videos' watchtimes and can be simply represented as Equation~\eqref{eqn:additive-reward}.
\end{example}

Another example of revenue optimization also adopts similar reward structure as in Example~\ref{ex:watchtime}.

\begin{example}[Revenue]\label{ex:revenue} In a scenario where the platform's goal is to maximize total revenue, user satisfaction outcome $Y_{t,j,k}\in\{0,R_j\}$ at time $t$ represents the revenue earned from user $X_t$ purchasing item $j$ priced at $R_j$. $Y_{t,j,k}$ equals to $R_j$ if a purchase occurs, and $0$ otherwise.
We employ a logistic model to capture purchase probability\footnote{
We assume an item's purchase likelihood is primarily influenced by its position, applicable in scenarios like online supermarkets where user budget and item substitution have minimal impact \cite{yao2021learning}.}:
$
\bbP(Y_{t,j,k}=R_j\given X_t;\alpha_j,\beta_j) = \big(1 + e^{-\alpha_j  k-\beta_j^T X_t }\big)^{-1}.
$
Note that since $X_t$ has an intercept coordinate, the item-specific parameter $\beta_j$ also captures potential impact of the price $R_j$ of item $j$.  
The aggregated user satisfaction of interest, which is the total expected revenue over $K$ items, has the form 
\begin{align*}
&r(X_t, \sigma)\\
=&\sum_{k=1}^K R_{q_t(k)}\bbP(Y_{t,q_k(t),k}=R_{q_k(t)} \given X_t;\alpha_{q_t(k)}, \beta_{q_t(k)}),
\end{align*}
which naturally gives a widely applied real-life example that supports the additive structure in our reward construction.
\end{example}

 The additive reward construction in Equation~\eqref{eqn:additive-reward} can be easily extend to a general additive reward form:
\[
r(X_t, \sigma) = H\Big(\sum_{k=1}^K g_k\big(\mu_{q_t(k)}(X_t, \sigma(k)\big)\Big),
\]
for some known increasing functions $H,g_1,\dots, g_K$. Such a general additive reward form also has a wide application, as the next example of click-through-rate shows.

\begin{example}[Click-Through-Rate]\label{ex:click-thr-rate}
The platform aims to maximize the user click probability on a list of $K$ items. We denote the user satisfaction outcome $Y_{t,j,k}\in\{0,1\}$ to indicate whether the item $j$ ranked at position $k$ is clicked by a user with feature $X_t$ at time $t$. We use a logistic model for the click probability:  $
   \bbP(Y_{t,j,k}=1 \given  X_t ; \alpha_j, \beta_j)  = \big(1 + e^{-\alpha_j  k-\beta_j^T X_t }\big)^{-1}
$\footnote{
We consider that the influence of other items on the click-rate of a particular item is wholly encapsulated by its position effect. This is applicable in scenarios such as ranking a concise list of short videos to maximize the total effective views, where the interference between items due to user time constraints is minimal \cite{yu2023improving}.}. The total user satisfaction is concerned with the click probability on the list of $K$ items as follows:
\begin{align*}
&r(X_t, \sigma)\\
=&1 - \prod_{k=1}^K\big(1- \bbP(Y_{t,q_t(k),\sigma(k)}=1 \given  X_t;\alpha_{q_t(k)},\beta_{q_t(k)})\big).
\end{align*}
That is, the aggregation functions $H(z)=1-\exp(-z)$ and $g_k(z)=-\log(1-z)$ for $k\in[K]$.
\end{example}

Next, we make the following standard assumptions on the outcome model.
\begin{assumption}[Regularity of  Outcomes]
\label{assump:regularity}
Assume that:
\begin{itemize}[itemsep=-2pt,leftmargin=15pt,topsep=-2pt]
    \item[(a)] $Y_{t,j,k}\in[0,R_0]$ for some known constant $R_0>0$, and $Y_{t,j,k} - \mu_i(X_t,k)$ is $\sigma^2$-sub-Gaussian.
\item[(b)] $A',H,\{g_k\}_{k\in[K]}\colon \RR\to \RR$ 
are non-decreasing.  
\item[(c)] The function $A'$ is twice differentiable with first and second order derivatives upper bounded by $M_{1}$ and $M_{2}$, respectively. It also satisfies $\kappa := \inf_{|z|\leq 1, |\theta-\theta_k|\leq 1} A''(\theta^\top z) >0$.
\item[(d)]  There exists a set of constants $\{c_k\}_{k\in [K]}$ such that for every $k\in[K]$, $H(\sum_{k=1}^K g_k(\mu_k))$ as a function of $\mu_k\in \RR^+$ is $c_k$-Lipschiz.
\end{itemize}
    
\end{assumption}
It is easy to see that the function $H(z),g_k(z)$ in Example~\ref{ex:click-thr-rate} obeys the above assumption with constants $c_k\equiv 1$ for all $k\in[K]$. Utilizing the basic additive reward structure in Equation~\eqref{eqn:additive-reward}, Example~\ref{ex:revenue} is a special case of the general additive reward form, with $H(z),g_k(z)$ being the identity function, which trivially satisfy Assumption~\ref{assump:regularity}.


\section{Upper Confidence Ranking: Adaptive Learning-to-Rank Algorithm}

In order to optimize cumulative regret, we follow the principle of ``optimism in the face of uncertainty''~\cite{hamidi2020general}, a strategy employed by UCB-typed algorithms \cite{lai1985asymptotically}. Specifically, when a user $X_t$ arrives, we estimate the upper confidence bound $U_t(X_t,\sigma)$ of the expected user satisfaction $r(X_t,\sigma)$ for each possible ranking $\sigma\in S_K$. We then select the ranking $\sigma_t$ that presents the largest upper confidence bound, which  strategy we refer to as \emph{Upper Confidence Ranking} (UCR):
\begin{align}
\label{eq:optimization_ranking}
    \sigma_t = \argmax_{\sigma\in S_k}\big\{U_t(X_t,\sigma) \big\}.
\end{align}
The challenge is twofold. 
The first challenge involves deriving high probability upper confidence bounds for $r(X_t,\sigma)$, which should be (i) uniformly valid across context space, permutation set, and experiment horizon to guarantee the statistical validity of these bounds; and (ii) converge rapidly to true user satisfaction scores for optimized cumulative regret.

The second challenge is to efficiently solve the optimization problem in~\eqref{eq:optimization_ranking}.
The original optimization problem to identify the optimal ranking requires enumerating  all possible rankings in $S_K$, which is NP-hard and leads to exponential computational time. Instead,
we leverage the reward model structure specified in Section~\ref{sec:reward-outcome-structure} and transform ranking problem into a bipartite matching problem, which can be solved via off-the-shelf graph algorithms in polynomial time. 
The complete procedure is summarized in Algorithm~\ref{alg:glm_ucb}.

\begin{algorithm}
\caption{Upper Confidence Ranking (UCR)}\label{alg:glm_ucb}
\begin{algorithmic}[1]
\REQUIRE Environment $\calE$, context sampling function $\calA_X$, reward generating function $\calA_R$, number of positions $K$, tuning parameter $\xi$, horizon $T$, randomization horizon~$T_0$.
\vspace{0.05in} 

\texttt{// Random initialization}
\FOR{$t=1,2,\dots,T_0-1$}
\STATE Observe context $X_t\sim\calA_X(\calE)$ and then randomly choose $K$ items $s_t(X_t)=(q_t(1),\dots, q_t(K))$ from $N$ items and order them randomly;
\STATE Sample $\sigma_t \sim \textrm{Unif}(S_K)$;
\STATE Take ranking $\sigma_t$ and observe outcomes $\{Y_{t,q_t(k),\sigma_t(k)}\}_{k\in[K]}\sim\calA_R(\calE, X_t, \sigma_t)$. 
\ENDFOR 
\vspace{0.05in} 

\texttt{// Upper Confidence Ranking}
\FOR{$t=T_0,\dots,T$}
\STATE Observe context $X_t\sim\calA_X(\calE)$;
\FOR{$j=1,\dots,N$}
\STATE Compute $\hat\theta_{t,j}=(\hat\alpha_{t,j}, \hat\beta_{t,j})$  via MLE as in \eqref{eqn:thetaest}.
\STATE Compute $V_{j}^{(t)}$ as in \eqref{eqn:Vj(t)}.
\ENDFOR
\STATE Obtain ranking $\sigma_t$ and $s(X_t)$ from Algorithm~\ref{alg:sub_ucb} with inputs  
$(\{\hat\theta_{t,j}\}_{j\in[N]}, X_t, \{V_{j}^{(t)}\}_{j\in[N]}, \xi)$.
\STATE Take ranking $\sigma_t$ and observe outcomes  $\{Y_{t,q_t(k), \sigma_t(k)}\}_{k\in[K]}\sim\calA_R(\calE, X_t, \sigma_t)$.
\ENDFOR
\ENSURE $\{(X_t, s_t(X_t), \sigma_t, Y_{t, q_t(1), \sigma_t(1)},\cdots,$\\$Y_{t, q_t(K), \sigma_t(K)})\}_{t\in[T]}$.
\end{algorithmic}
\end{algorithm}


\begin{algorithm}
\caption{Subroutine:  Upper Confidence Ranking via Maximum Weighted Bipartite Matching}\label{alg:sub_ucb}
\begin{algorithmic}[1]
\REQUIRE Parameter $\{\hat\theta_{t,j}\}_{j\in[N]}$, context $x$, covariances $\{\covmat{t}{j}\}_{j\in[N]}$, tuning parameter $\xi$.
\vspace{0.05in} 
\STATE Compute augmented feature $z_k = (k,x)$ for $k\in[K]$.
\FOR{$(k,j)=\{1,\dots,K\}\times\{1,\dots,N\} $}
\STATE Compute $w^U_t(j,k) := g_k(A'(\hat{\theta}_{t,j}^T z_k + \xi \cdot \|z_k\|_{(\covmat{t}{j})^{-1}} ))$
\ENDFOR
\STATE Obtain solution $\hat{m}(j,k)$  from the maximum weight imperfect matching~\eqref{eq:mwm_general} with $w^U_t(j,k)$. 
\STATE Set $\sigma_t(k) = \sum_{j=1}^N j\cdot\ind\{\hat{m}_t(j,k)=1\}$.  
\vspace{0.05in}
\ENSURE Ranking $\sigma_t$ with the retrieved set $s(X_t)=\{j\in[N]:\sum_{k\in[K]}\hat{m}_t(j,k)=1\}$.
\end{algorithmic}
\end{algorithm}

\subsection{Constructing Upper Confidence Bounds} 

At time $t$, for a user $X_t=x$, we construct upper confidence bounds of $r(x,\sigma)$ for each ranking $\sigma\in S_K$. We achieve this by deriving the upper confidence bounds of the user satisfaction score $\mu_j(x, k)$ for each item $j\in s(x)$ and each position $k\in[K]$, using a technique adapted from \cite{li2017provably}.


To construct upper confidence bounds, the algorithm needs two phases: (i) \emph{random initialization phase} to collect enough information for constructing initial upper confidence bounds; and (ii) \emph{upper confidence ranking phase} where the algorithm actually learns and optimizes the ranking strategy. During the random initialization phase, each item will be retrieved with an equal probability under any context $X_t\in\mathbf{R}^{d}$. For our upper confidence bound framework,  it is necessary to ensure that we collect enough information to empower the upper confidence ranking phase.

In specific, at time $t$, a user comes with context $X_t=x$. 
First, the algorithm adopts $T_0$ rounds of random initialization, where after a user comes, the algorithm randomly selects $K$ out of the $N$ items as the recommended list, randomly ranks them, and collect the responses $Y_{i,j,k}$ for $k=1,\dots,K$ and $j\in s_t(X_t)$. 
After the random initialization, we use the UCR approach described as follows. 
For each item $j\in[N]$, we use observations up to time $t-1$ to estimate the item-specific parameter via maximum likelihood estimation (MLE) 
$\hat{\theta}_{t,j}:=(\hat{\alpha}_{t,j}, \hat{\beta}_{t,j})$:
\begin{align}\label{eqn:thetaest}
&\hat{\theta}_{t,j}=\argmax_{(\alpha, \beta)}\Big\{\sum_{\tau\in[t-1]: j\in s(X_\tau)} Y_{\tau,j,\sigma_{\tau}({q^{-1}_\tau(j)})}\nonumber\\
&\big(\alpha \sigma_\tau(q^{-1}_\tau(j)) + \beta^T X_\tau\big)  - A\big(\alpha\sigma_\tau(q^{-1}_\tau(j))+ \beta^T X_\tau\big)\Big\},
\end{align}
where here $s(X_\tau)$ is actually chosen by our algorithm  (which we shall describe shortly). We similarly construct the upper confidence bound of $\mu$ as
\$
\muhatupp{t}{j}(z) := A'(\hat{\theta}_{t,j}^T z + \xi \cdot \|z\|_{(\covmat{t}{j})^{-1}} )
\$
with covariance matrix
\begin{align}\label{eqn:Vj(t)}
    V_{j}^{(t)} :=& \sum_{\tau=1}^t \ind\{j\in s(X_\tau)\}\cdot z_{\tau ,j} z_{\tau,j}^\top,\nonumber\\
\quad z_{\tau,j} =&\big(\sigma_\tau\big(q_\tau^{-1}(j)\big), X_\tau\big).
\end{align}
We  then have  the  upper confidence bound of $r(x,\sigma)$:
\[
\rupp{t}(x,\sigma) :=   H\Big(\sum_{k=1}^K g_k\big(\muhatupp{t}{q_t(k)}(x, \sigma(k)\big)\Big).
\]

\subsection{Upper Confidence Ranking via Maximum Weighted Bipartite Matching}
We now describe how to solve the  ranking in \eqref{eq:optimization_ranking} by reformulating  it as a bipartite maximum weight matching problem, leveraging the generalized additive form of $r$. At each time  $t$, the bipartite graph $G^U_t=(V_t, E^U_t)$ is constructed as:
\begin{itemize}[itemsep=-1pt,leftmargin=15pt,topsep=-2pt]
\item Nodes $V_t$: $N$ left-side nodes of items $[N]$, and $K$ right-side nodes of positions $[K]$;
\item Edges $E^U_t$:   edge $(j,k)$ with weight $w^U_t(j,k)=g_k(\hat{\mu}_{t,j}^U(k,X_t))$, for $(j,k)\in [N]\times [K]$.
\end{itemize}
Then, we consider the following maximum weight  matching problem on the bipartite graph $G_t^U$:
\begin{equation}
\label{eq:mwm_general}
    \begin{aligned}
        \max_{m_t} \quad &\sum_{j\in [N], k\in [K]} w^U_t(j,k)m_t(j,k)\\
        \mbox{s.t.}\quad & \sum_{j\in [N]} m_t(j, k)=1, \quad \forall k \in [K]\\
        & \sum_{k\in [K]} m_t(j, k)\leq 1, \quad \forall j \in [N]\\
        & m_t(j,k)\in \{0,1\}, \quad \forall j \in [N], \forall k \in [K],
    \end{aligned}
\end{equation}
where $w_t^{U}(j,k)$ is calculated as in Line~3 of Algorithm~\ref{alg:sub_ucb}, for every $(j,k)\in[N]\times[K]$.

Problem  \eqref{eq:mwm_general} can be solved using the Hungarian algorithm, which, based on a primal-dual formulation, achieves a solution in $O(K^2\log K)$ time \cite{ramshaw2012weight,kuhn1955hungarian}. 
There exists a one-to-one correspondence between the solution $m_t$ of \eqref{eq:mwm_general} and the solution $\sigma_t$ of \eqref{eq:optimization_ranking}:
\begin{equation}
    \sigma_t(j) = k\quad\Leftrightarrow\quad m_t(j,k)=1,
\end{equation}
and the retrieved set 
\[
s_t(X_t)=\{j\in[N]:\sum_{k\in[K]} m_t(j,k)=1\}.
\]

\begin{remark}
\label{remark:ucr_greedy}

An alternative method (G-MLE) uses a greedy strategy, ranking items based on their MLE of user satisfaction score instead of the upper confidence bound. This strategy also forms a bipartite graph $G_t$, similar to $G^U_t$, but uses the MLE to assign edge weights: for each item $j \in s_t(X_t)$ and each position $k \in [K]$, the weight $w_t(j,k)$ is given by $g_k(\hat{\mu}_{t,j}(k,X_t))$. A comparison of the bipartite graphs resulting from UCR and the MLE ranking is illustrated by  Figure~\ref{fig:comparison} in Appendix~\ref{appx:exp-details}. We shall compare the performance of UCR against G-MLE in Section~\ref{sec:exps}, where G-MLE serves as a benchmark.

\end{remark}

\section{Main Result on Cumulative Regret}
\label{sec:regret}

In this section, we present our main theoretical results on the cumulative regret guarantees using our algorithm. 
We highlight the key steps in our proof in Section~\ref{subsec:sketch} and defer the complete version to  Appendix~\ref{app:thm_general}.





\begin{theorem}\label{thm:n-choose-k-regret}
Fix any $\delta\in(0,1)$, and let $c_1:=\min\{\frac{1}{2K},c_x\}>0$. Suppose Assumption~\ref{assump:context} and \ref{assump:regularity} hold, and $T_0 \geq \max\big\{  (\frac{16}{3c_1}+\frac{32(K+N)^2}{N^2c_1})\log\frac{2(d+1)}{\delta} , ~\frac{6\bar\sigma^2}{c_1\kappa^2} ( (d+1) \log(1+2T/d) + \log(1/\delta) )  \big\}$. Then with probability at least $1-\delta$,
\[
R_T\leq R_0 T_0 + \frac{5\bar{\sigma}}{\kappa}\cdot  M_1\cdot \bar{c} \cdot d\sqrt{NK T } \log(T/(d\delta)),
\]
where $\bar{c}=\max{k\in[K]}c_k$. With the proper choice of the initialization phase $T_0$,
\[
R_T\leq\tilde{O}\bigg(\frac{(K+N)^2}{c_1}+\frac{\bar{\sigma}}{c_1\kappa^2}\cdot d+\frac{\bar{\sigma}}{\kappa}\cdot M_1\cdot\bar{c}\cdot d\sqrt{NKT}\bigg).
\]
\end{theorem}

Theorem~\ref{thm:n-choose-k-regret} shows that the regret of our algorithm scales with $\sqrt{T}$, which is minimax optimal in standard contextual bandit results \cite{agrawal2012analysis,chu2011contextual}. 
The factor $d$ is similar to the GLM bandit result in~\cite{li2017provably}. 
From the factor $\sqrt{NK}$, 
we see that our algorithm overcomes the combinatorial complexity of the ranking space (which is originally $\binom{N}{K}=\frac{N!}{K!(N-K)!}$ 
for choosing $K$ retrieved items from a total of $N$ items). 

The regret bound in the special case $N=K$ is provided as Corollary~\ref{thm:n=k}. In this case, we obtain a slightly different constant factor $c_H =\sum_{k=1}^K c_k$ instead of $\sqrt{NK}\max_{k\in [K]}c_k$, because all items appear in the ranked list, and each item incurs an estimation error that enters the reget bound. 

\begin{corollary}\label{thm:n=k} Let $N=K$. 
Fix any $\delta\in(0,1/2)$. 
Under Assumption, suppose 
$T_0 \geq \max\big\{ (\frac{32}{3c_1}+\frac{256}{c_1^2} \log(\frac{4d+4}{\delta}) ,~\frac{6\bar{\sigma}^2}{c_1\cdot\kappa^2}((d+1) \log(1+2T/d) + \log(2/\delta)) \big\}$. Then   
\$
R_T &\leq R_0 T_0 + \frac{5\bar{\sigma}}{\kappa}\cdot M_1\cdot c_H \cdot d\sqrt{ T } \log(T/(d\delta)) 
\$
with probability at least $1-\delta$, where we denote $c_H :=\sum_{k=1}^K c_{k}$. 
\end{corollary}

\subsection{Proof sketch of Theorem~\ref{thm:n-choose-k-regret}}
\label{subsec:sketch}
In this part, we lay out the proof sketch for Theorem~\ref{thm:n-choose-k-regret}. 
In a nutshell, our theory proceeds in three steps: (i) the random initialization of $T_0$ steps ensures that the covariance matrices $\covmat{t}{j}$ are well-conditioned; (ii) the estimation error $\hat\theta_{t,j}-\theta_j$ is small once $\covmat{t}{j}$ is well-conditioned; (iii) the cumulative regret is bounded in terms of the Lipschiz constant in the mean reward functions and the parameter estimation error. 
Among the three steps, (i) follows from standard concentration inequalities, which we defer to Lemma~\ref{lemma:eigenvaluebound-nchoosek} 
in Appendix~\ref{app:thm_general}. 

The key step (ii) is summarized in Proposition~\ref{prop:glm_err-n-choose-k}, which shows that with sufficiently many random initialization samples, the estimation errors later on can be uniformly bounded for all $t\geq T_0$. Our proof technique adapts that of~\citep{li2017provably}, and we include the  detailed proofs in Appendix~\ref{app:prop_general}. 

\begin{prop}\label{prop:glm_err-n-choose-k}
For any $\delta\in(0,1)$. Let $c_1:=\min\{\frac{1}{2K},c_x\}>0$ and let 
\begin{align*}T_0\geq\max\big\{&  (\frac{16}{3c_1}+\frac{32(K+N)^2}{N^2c_1})\log\frac{2(d+1)}{\delta} ,\\
&\frac{6\bar\sigma^2}{c_1\kappa^2} ( (d+1) \log(1+2T/d) + \log(1/\delta) )  \big\}.
\end{align*}
Then with probability at least $1-\delta$, for all $t\in[T_0,T]$ and all $j\in[N]$, it holds that
\[
\|\hat\theta_{t,j}-\theta_j\|_{V_j^{(t)}}
\leq \frac{\sqrt{3} \bar\sigma}{\kappa} \sqrt{ (d+1) \log(1+2T/d) + \log(1/\delta)}.
\]
\end{prop}

In step (iii), we bound the regret with the following lemma on the event that the estimation error $\hat\theta_{t,j}-\theta_j$ is uniformly small. Recall that $z_{\tau,j} = \big(\sigma_\tau\big(q_\tau^{-1}(j)\big), X_\tau\big)$ is the aggregated feature for item $j$ at time $\tau$. 
The proof of Lemma~\ref{lem:regret_general-nchoosek} is in Appendix~\ref{app:regret_lemma}. 

\begin{lemma}\label{lem:regret_general-nchoosek} 
Denote
\[
\xi = \frac{\sqrt{3} \bar\sigma }{\kappa } \sqrt{ (d+1) \log(1+2T/d) + \log(2/\delta)}.
\]
On the event of Proposition~\eqref{prop:glm_err-n-choose-k}, it holds that 
\[
R_T\leq R_{T_0} + M_1\cdot 2\xi \cdot \sum_{t=T_0}^T \sum_{k=1}^K c_k  \cdot \|Z_{t,q_t(k)}\|_{(V^{(t)}_{q_t(k)})^{-1}}.
\]
\end{lemma}

In Lemma~\ref{lem:regret_general-nchoosek}, the regret is split into $R_{T_0}$ (due to initial random initialization) and the estimation error. For the estimation error,  each recommended item $q_t(k)\in s_t(X_t)$ has regret bounded by its inverse-covariance-normalized feature norm, which can be further controlled via deterministic bounds on self-normalized norms~\citep{abbasi2011improved}. 

Finally, we can conclude Theorem~\ref{thm:n-choose-k-regret} by combing the results of Lemma~\ref{lemma:eigenvaluebound-nchoosek}, \ref{lem:regret_general-nchoosek} and Proposition~\ref{prop:glm_err-n-choose-k}. Observing the fact that on the event in Proposition~\ref{prop:glm_err-n-choose-k}, $\lambda_{\min}(V_s^{(T_0)})\geq 1$ 
while $\|Z_{t,s}\|\leq 1$ for all time $t$ and item $s$, by a helper Lemma~\ref{lem:self_normal}, we derive the desired results. More details are in Appendix~\ref{app:thm_general}.


\begin{figure*}[h]
\begin{multicols}{3}
    \includegraphics[width=\linewidth]{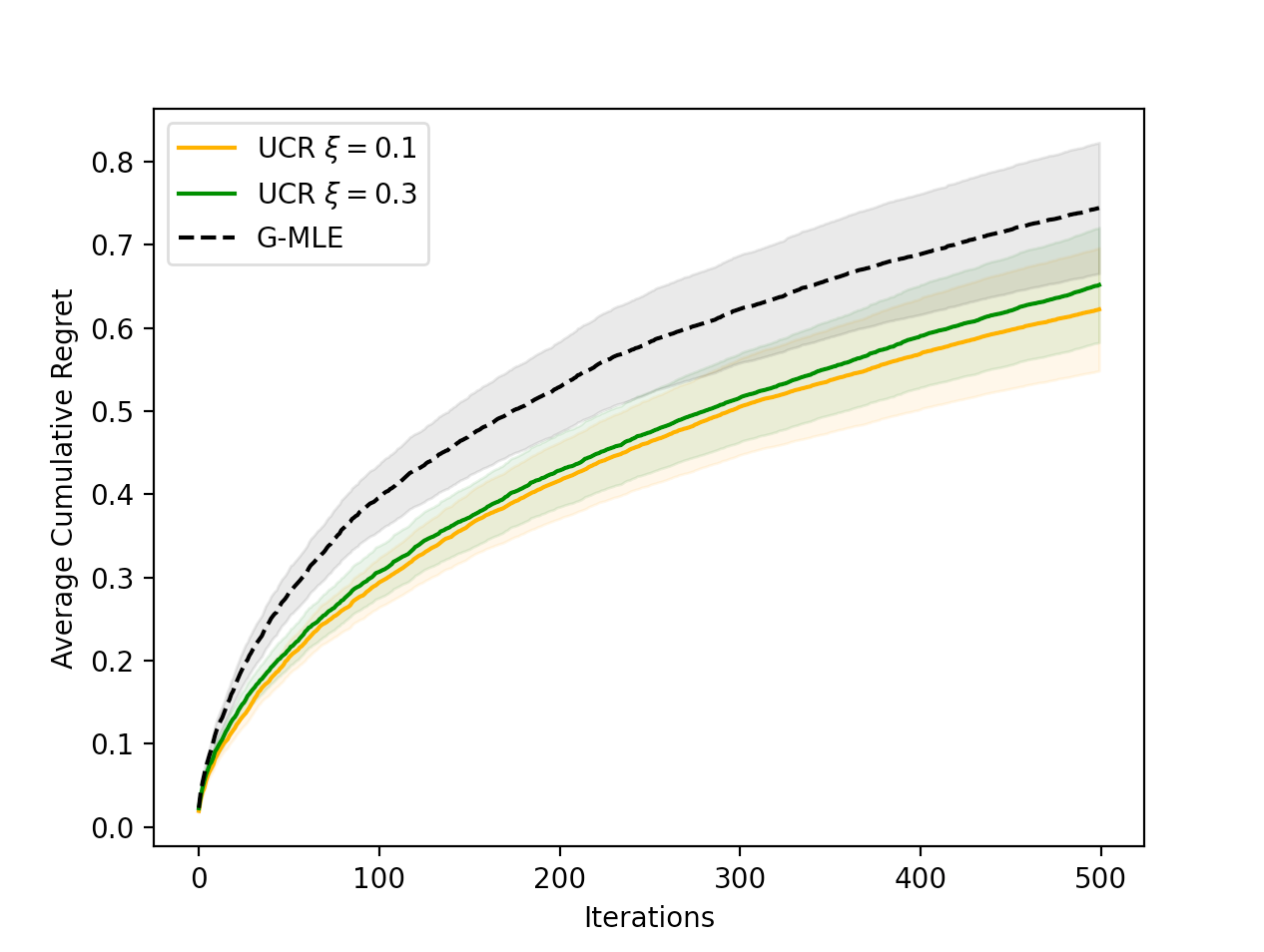}\par 
    \includegraphics[width=\linewidth]{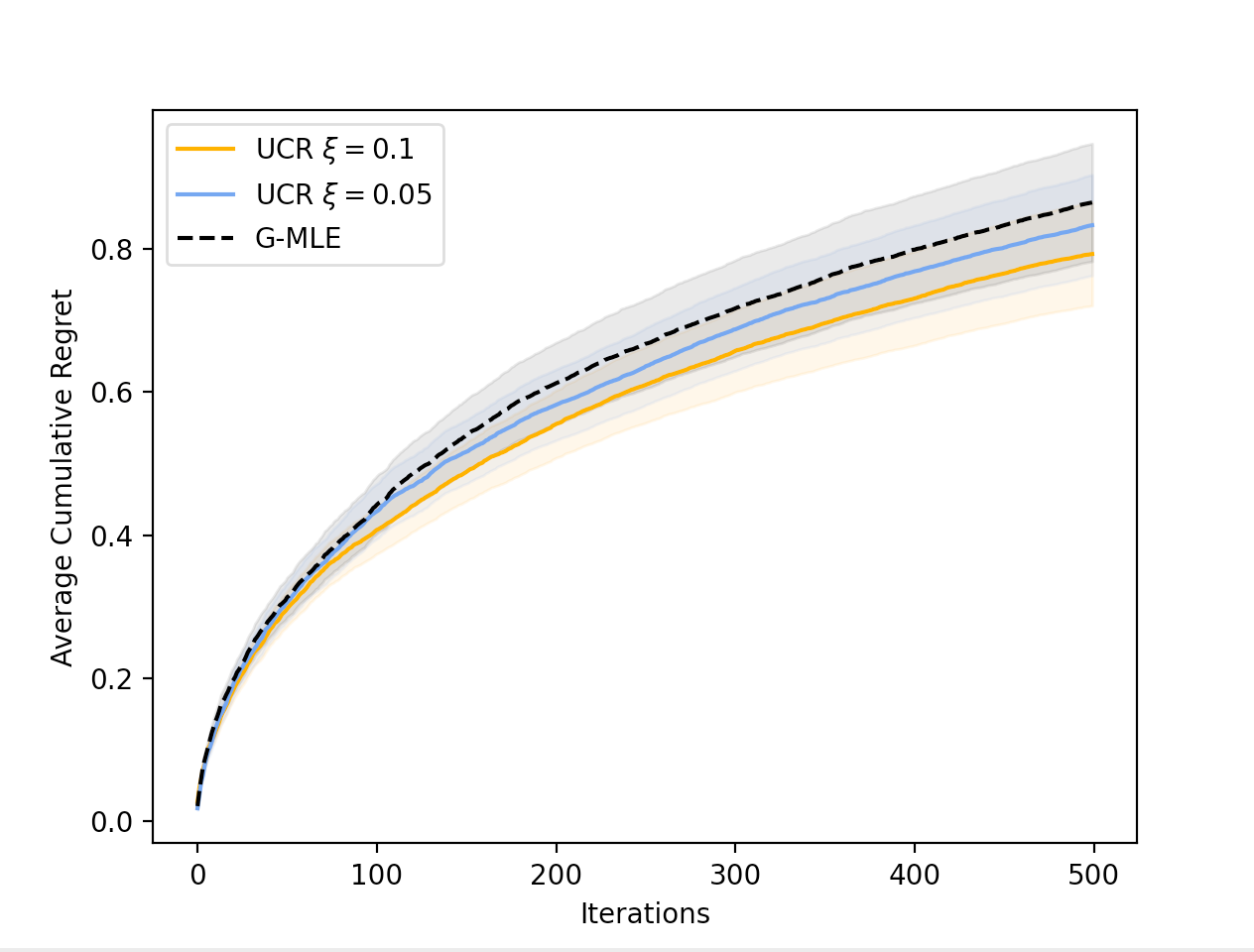}\par
    \includegraphics[width=\columnwidth]{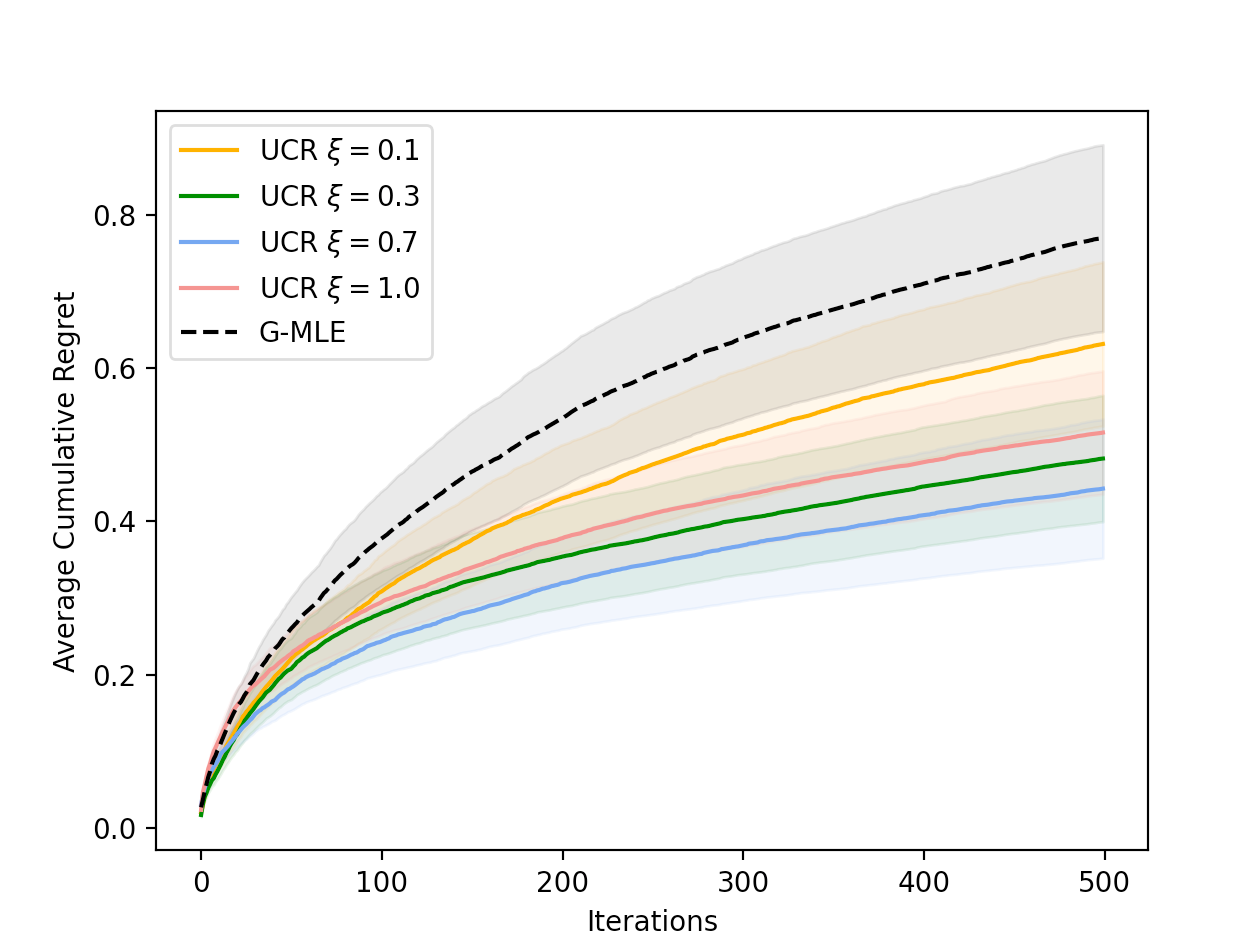}\par
    \end{multicols}
\caption{The average cumulative regret (with standard variation interval) of UCR and G-MLE in the simulated environment. The figure on the left is the result of the $N=7,K=5$ case; in the middle is the result of the $N=10,K=5$ case; the figure on the right is the result of the $N=K=5$ case.}
\label{fig:simu-results}
\end{figure*}

\begin{figure}[h]
\vskip 0.2in
\begin{center}
\centerline{\includegraphics[width=\columnwidth]{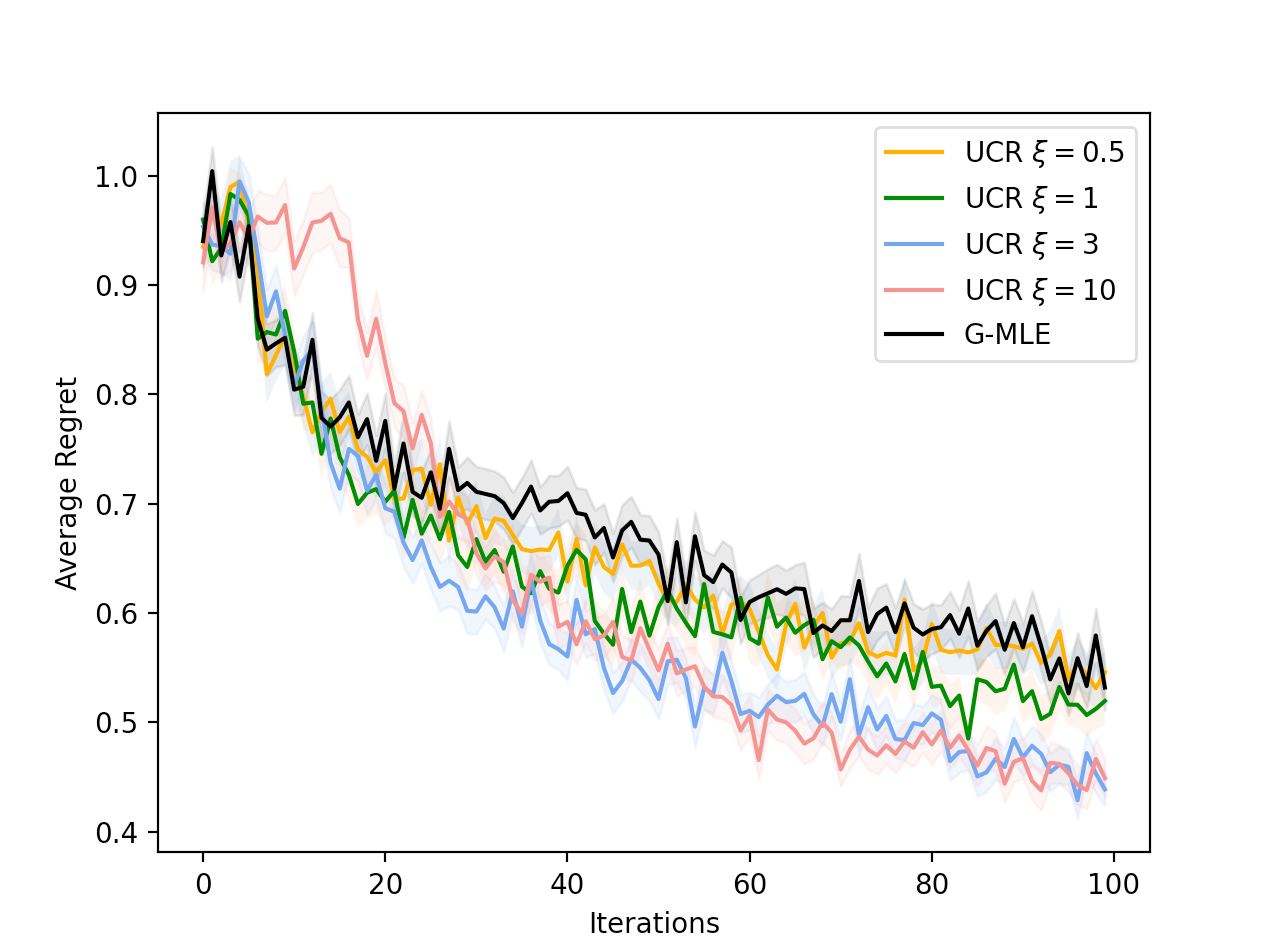}}
\caption{Average relative regret (with standard variation interval) of UCR and G-MLE on the real-world dataset. }
\label{fig:realworld-results}
\end{center}
\vskip -0.2in
\end{figure}

\section{Experiments}\label{sec:exps}

We hereby provide empirical performances of UCR and G-MLE (the comparison of which is in Remark~\ref{remark:ucr_greedy}) on \emph{a simulated dataset} and \emph{a real-world dataset}, with the goal of illuminating how the two algorithms perform across different environments. We note the inclusion of the real-world dataset to test our algorithm's robustness and potential effectiveness in practical applications, as simulated environments often present idealized scenarios.

\textbf{Simulated Environment:} We evaluate the ranking algorithms on a generated environment, where the ground truth parameters are sampled as follows: the positional effects for each item are drawn uniformly from the interval $[0, 1]$, and the contexts are drawn uniformly from a norm ball of a specified radius. We then perform ranking tasks within this generated environment. The detailed environment generation procedure is outlined in Algorithm~\ref{alg:generate_environment} and \ref{alg:generate_context} in Appendix~\ref{app:alg_proc}\footnote{The python code for executing the experiment can be found in https://github.com/arena-tools/ranking-agent.}. We run the experiments under two main settings, one with $N>K$ and another with $N=K$. For the first setting, we includes two cases of $N=7,K=5$ and $N=10,K=5$. For the second setting, we let $N=K=5$.

\textbf{Real-world Dataset:} We test UCR and G-MLE using a real-world task, with the goal to maximize click-through rates of the company’s product recommendations \footnote{For data-privacy reasons, the name of the company is not disclosed.}. An offline dataset from the historical data, collected via a heuristics-based control policy, is used to learn a simulator. This simulator generates clicks for a given set of rankings, which are formulated as the reward in Algorithm~\ref{alg:generate_click} of Appendix~\ref{appx:exp-details}. The dataset comprises 13,717 samples and 436 unique items, all used for simulator training. Bootstrap samples are taken from a subset of 259 samples and $N=$114 items with positive rewards. Most features indicate if a user has recently purchased a similar item. The summary statistics of these features are in Table~\ref{tbl:crouch} of Appendix~\ref{appx:exp-details}. We run UCR and G-MLE with $K=3$ retrieved items and update each iteration with a batch size of 30 (each batch consists of 3 items), repeatedly for 200 runs.

\subsection{Empirical Results}
We present the empirical results of cumulative regrets on simulated dataset in Figure~\ref{fig:simu-results}, over $T=500$ iterations and 300 runs. Additionally, we present the empirical results of the relative regrets on real-world dataset in Figure~\ref{fig:realworld-results}, over $T=100$ iterations and 200 runs. All experiments have a initialization phase $T_0=5$. For each setting we run several upper confidence parameters $\xi$ of UCR and present the regret curves of these instances along with those of the baseline G-MLE approach.

\paragraph{UCR consistently outperforms the G-MLE approach across different environments.} As shown in Figure~\ref{fig:simu-results}, in both $N>K$ and $N=K$ settings, UCR yields lower average cumulative regret. These results also demonstrate the importance of choosing the right upper confidence parameter $\xi$. While UCR generally outperforms MLE, the advantage can shrink with poor choices of parameter $\xi$. Despite this, UCR still outperforms the baseline G-MLE in all settings and with all reasonably chosen $\xi$, showing the algorithm’s robustness.

\paragraph{UCR maintains its advantage over G-MLE on real-world applications.} Figure~\ref{fig:realworld-results} shows that the average relative regret of UCR on the real-world dataset is lower than that of the baseline MLE approach across all instances with various chosen hyperparameters~$\xi$. Similarly as in the simulated environment experiment, we explore how the hyperparameter~$\xi$ would affect the performance of UCR. Figure~\ref{fig:realworld-results} indicates that poor choices of the hyperparameter $\xi$ could overcome the advantages of UCR over G-MLE. In this case, UCR with $\xi=0.5$ is comparable to G-MLE for large iterations; while lager $\xi$ results in smaller average relative regrets.

Overall, the results show that UCR not only outperforms G-MLE in a simulated environment but also in real-life applications, where the tasks inevitably contain more noise, and conclude the advantages of UCR further.


\section*{Acknowledgements}
We cordially thank Pratap Ranade and Engin Ural for providing a 
world class environment at Arena that has helped shape and
push an ambitious vision of this research agenda, where cutting edge active learning is finding its way to economic domains, for which this work is a specific instance.

Ruohan Zhan is partly supported by the Guangdong Provincial Key Laboratory of Mathematical Foundations for Artificial Intelligence (2023B1212010001) and the Project of Hetao Shenzhen-Hong Kong Science and Technology Innovation Cooperation Zone (HZQB-KCZYB-2020083). Zhengyuan Zhou also gratefully acknowledges the 2024 NYU CGEB faculty grant.

\section*{Impact Statement}
This paper presents work whose goal is to advance the field of Machine Learning. There are many potential societal consequences of our work, none of which we feel must be specifically highlighted here.

\nocite{langley00}

\bibliography{mybib}
\bibliographystyle{icml2024}

\newpage
\appendix
\onecolumn
\section{Additional Experiment Details}\label{appx:exp-details}
We shall introduce the details of the experiment. 
First of all, to better distinguish the bipartite matching schemes of UCR (our proposed algorithm) and G-MLE (the benchmark algorithm), we present the visualization of both in Figure~\ref{fig:comparison}.

\begin{figure}[ht]
\vskip 0.2in
\begin{center}
\centerline{\includegraphics[width=\columnwidth]{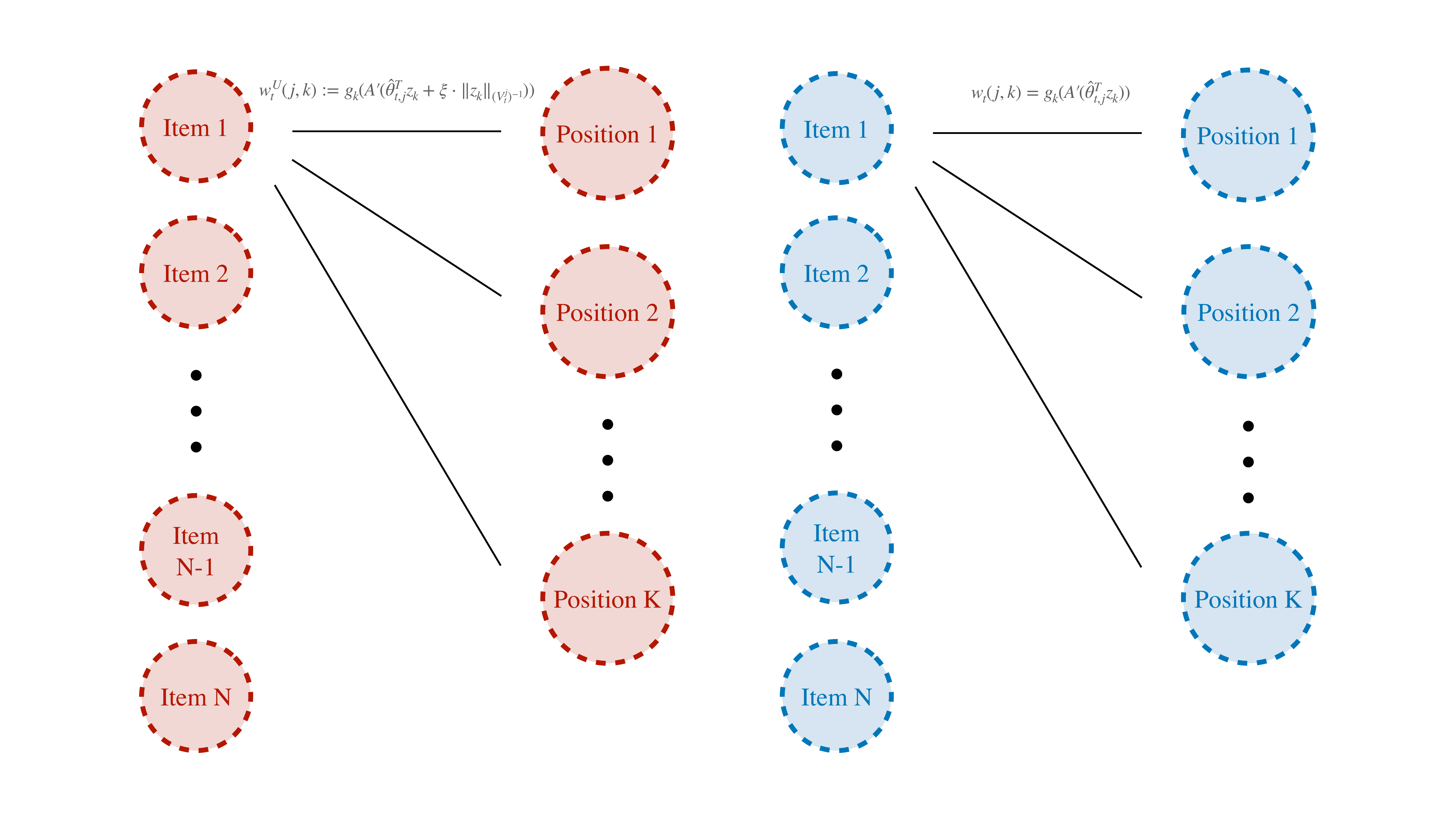}}
\caption{Visualization of Bipartite Matching of UCR and greedy MLE approach.}
\label{fig:comparison}
\end{center}
\vskip -0.2in
\end{figure}

The details for how the simulated environment are listed below. We generate the ground truth simulator parameters following Algorithm \ref{alg:generate_environment}, and for each step context is generated following Algorithm \ref{alg:generate_context} from which the algorithm makes an update and predicts the probability of clicks. In our experiments, the context dimension is set to be $d=7$.

\label{app:alg_proc}
\begin{algorithm}[H]
    \caption{Generate simulator parameters}
    \label{alg:generate_environment}
    \begin{algorithmic}[1]
        \REQUIRE number of items $N$, dimension of features $d$.
        \FOR{$n=1,\dots,N$}
        \STATE $\alpha_n\leftarrow \mbox{Uniform}[0,1].$
        \STATE $\beta_n\leftarrow \mbox{Uniform}(\calB(1,\bbR^d))$.
        \ENDFOR
        \ENSURE position effect $\{\alpha_n\}_{n=1}^N$, item features $\{\beta_n\}_{n=1}^N$.
    \end{algorithmic}
\end{algorithm}

\begin{algorithm}
    \caption{Generate context}
    \label{alg:generate_context}
    \begin{algorithmic}[1]
        \REQUIRE Context dimension $d$.
        \STATE $x\leftarrow \mbox{Uniform}(\calB(1,\bbR^d))$.
        \ENSURE $x$.
    \end{algorithmic}
\end{algorithm}

The details for the construction of the real-world experiment are as follows. We first provide the summary statistics for the real-world dataset used in our experiments in Table~\ref{tbl:crouch}. All features except Store Type are binary features indicating if the user has purchased the same type of item recently.

\begin{table}[H]
\centering
    \caption{Summary Statistics for Offline Dataset}
    \label{tbl:crouch}
    \begin{tabular}{  l  p{10.5cm}  p{1.2cm} }
        \toprule
\textbf{Features}      
& \textbf{Description}   
& \textbf{p-Value} \\\midrule
Store Type
& Represents the type of store making the purchase of the drink        
& 0       
\\\hline
Item        
& Indicated if the user recently purchased the same item
& 0 \\\hline
Liquid &
Indicates if the user recently purchased the same liquid type  
& 0.792 \\\hline
Style        
& Indicates if the user recently purchased same style of items
& 0.005 \\\hline
Brand       
& Indicates if the user recently purchased the same brand of items 
& 0.177 \\\hline
Container      
& Indicates if the user recently purchased any items of the same container type
& 0.008 \\\hline
Material      
& Indicates if the user recently purchased any items of the same container material
& 0.489 \\\hline
Case      
& Indicates if the user recently purchased any items of the same container case configuration
& 0.684 \\
        \bottomrule
    \end{tabular}
\end{table}

We simulate the reward under this setting as occurrence of clicks, which is outputted by the model as in Algorithm~\ref{alg:generate_click}.

\begin{algorithm}
    \caption{Generate click occurrence}
    \label{alg:generate_click}
    \begin{algorithmic}[1]
        \REQUIRE true environment $\{(\alpha_k, \beta_k)\}_{k=1}^K$, context $x$, ranking $\sigma$.
        \FOR{$k=1,\dots,K$}
        \STATE Sampling probability $p_k = \frac{1}{\exp(\alpha_k\sigma(k) - x^T\beta_k)}$.
        \STATE Click $Y_k \sim \mbox{Bernoulli}(p_k)$
        \ENDFOR
        \ENSURE Clicks for $K$ items $\{Y_k\}_{k=1}^K$.
    \end{algorithmic}
\end{algorithm}

\section{Extension to Feature-Based Ranking}
\label{sec:extension_feature} The UCR algorithm can be extended to perform feature-based ranking by making the model only dependent on user and product features, which allows the UCR ranking algorithm to learn the effects of interactions between the user and product and rank without prefiltering and rank an arbitrary set of items with only a single ranking model. Model the probability of clicks as 
\begin{align*}
   \bbP(Y_{t,j,k}=1 \given  X_t ; \alpha_j, \beta_j) = \bbE[Y_{t,j,k}|X_t ; \alpha_j, \beta_j] = \big(1 + e^{-\vec{z}^T(k)\vec\theta})^{-1}.
\end{align*}
$\vec\theta=[-\alpha_1, \alpha_2^T, \alpha_3^T, w_{11}, w_{12}, ..., w_{mn}]^T$ is the feature vector representing each of the product and user features and flattened correlation matrix $W$.

$\vec{z}(k)=[\sigma_i(k), \vec{x_i}, \vec{y_i}, \vec{x}_i\times \vec{y}_i]$ where $\vec{x}_i\times \vec{y}_i$ is the flattened Cartesian product of the product and user vectors at step $i$ and $\sigma_i(k)$ is the rank of item $k$ at step $i$

Then, the UCB probability can be expressed as $p(k, i)=\{1+\exp(-\alpha_1\sigma_i(k)+\alpha_2^Tx_i+ \alpha_3^Ty_i+x^T_i W y_i- 3\xi\sqrt{z^TV^{-1}z}\}^{-1}, z=(i, \vec{x_i}, \vec{y_i}, \vec{x}_i\times \vec{y}_i).$

The solution can be obtained by solving the matching problem for the bipartite graph in the same way as stated in Algorithm~\ref{alg:sub_ucb}.



\section{Omitted Technical Proofs}

\subsection{Derivation of Equation~\eqref{eq:glm_model}}
Denote the log-likelihood function $l(X_t,\beta_j,\alpha_j;y)=\log f_{Y_{t,j,k}\mid X_t; \beta_j,\alpha_j}(y)$ as a function of $(X_t;\beta_j,\alpha_j)$ and $y$. The mean of $Y_{t,j,k}\mid X_t; \beta_j,\alpha_j$ can be derived from the known relations of exponential family: $\mathbb{E}[\frac{\partial l(X_t,\beta_j,\alpha_j;y)}{\partial (\alpha_j k + \beta_j^T X_t)}]=0$. Define $\hat{h}(Y_{t,j,k},\tau)=\exp(h(Y_{t,j,k},\tau))$, we have that
\begin{align*}
l(X_t,\beta_j,\alpha_j;y)=\frac{y(\alpha_j k + \beta_j^T X_t) - A(\alpha_j k + \beta_j^T X_t)}{d(\tau)}+\hat{h}(y,\tau),
\end{align*}
and therefore
\begin{align*}
\frac{\partial l(X_t,\beta_j,\alpha_j;y)}{\partial (\alpha_j k + \beta_j^T X_t)}=\frac{y-A'(\alpha_j k + \beta_j^T X_t)}{d(\tau)}.
\end{align*}
Plugin the result back into $\mathbb{E}[\frac{\partial l(X_t,\beta_j,\alpha_j;y)}{\partial (\alpha_j k + \beta_j^T X_t)}]$, we have that
\begin{align*}
0=\mathbb{E}[\frac{\partial l(X_t,\beta_j,\alpha_j;y)}{\partial (\alpha_j k + \beta_j^T X_t)}]=\frac{\mathbb{E}[Y_{t,j,k}\mid X_t;j,k]-A'(\alpha_j k + \beta_j^T X_t)}{d(\tau)},
\end{align*}
hence
\begin{align*}
\mathbb{E}[Y_{t,j,k}\mid X_t;j,k]=A'(\alpha_j k + \beta_j^T X_t).
\end{align*}
For more details, please refer to \cite{mccullagh2019generalized}.

\subsection{Proof of Theorem~\ref{thm:n-choose-k-regret}}

\label{app:thm_general}

\begin{proof}[Proof of Theorem~\ref{thm:n-choose-k-regret}]

Step (i) in our proof sketch is in Lemma~\ref{lemma:eigenvaluebound-nchoosek}, whose proof is 
in Appendix~\ref{sec:eigenvaluebound-nchoosek}. 

\begin{lemma}\label{lemma:eigenvaluebound-nchoosek}
Fix any $\delta\in(0,1)$, and let $c_1:=\min\{\frac{1}{2K},c_x\}$ and $B>0$ be any positive constant. Suppose $T_0\geq\max\{(\frac{16}{3c_1}+\frac{32(K+N)^2}{N^2c_1})\log\frac{2(d+1)}{\delta},\frac{2B}{c_1}\}$, then with probability at least $1-\delta$, we have $\lambda_{\min}(V_j^{(t)})\geq B$ for all $t\geq T_0$ and all $j\in[N]$.
\end{lemma}

Let $k_{t,j}$ be the position assigned to item $j$ at $t$ 
if item $j$ is included in the recommended list. Then following the upper bound in Lemma~\ref{lem:regret_general-nchoosek}, we have
\begin{align*}
R_T-R_{T_0}\leq \sum_{t=T_0}^T \sum_{j=1}^N \ind\{j\in s_t(X_t)\} \cdot  c_{k_{t,j}}  \cdot \|Z_{t,j}\|_{(V^{(t)}_j)^{-1}} 
\leq \bar{c}   \cdot \sum_{j=1}^N\sum_{t\in \cT_j,t\geq T_0}   \|Z_{t,j}\|_{(V_j^{(t)})^{-1}} ,
\end{align*}
where $\bar{c}=\max_{k}c_k$. 
%
We then use the self-normalized concentration inequality (c.f.~Lemma~\ref{lem:self_normal}) to bound $\sum_{t\in \cT_j,t\geq T_0}   \|Z_{t,j}\|_{(V_j^{(t)})^{-1}}$. 
Under the notations of Lemma~\ref{lem:self_normal}, 
fixing any $j\in [N]$, we let
$\bar\cT_j = \{s\colon s\in \cT_j,s\geq T_0\} = \{s_{j,1},s_{j,2},\dots,s_{j,|\bar\cT_j|}\}$; that is, $s_{j,t}$ is the $t$-th time after $T_0$ that item $j$ appears in the  recommended list. 
Then setting $X_t = Z_{s_{j,t},j}$, 
$V= V_j^{(T_0)}$, we note that 
$\bar{V}_{t} := V_j^{(s_{j,t})}=V + \sum_{i=1}^t X_{s_{j,t}}X_{s_{j,t}}^\top$. 
Also, on the event in Proposition~\ref{prop:glm_err-n-choose-k}
we see $\lambda_{\min}(V_s^{(T_0)})\geq 1$ 
while $\|Z_{t,s}\|\leq 1$. 
Thus, invoking Lemma~\ref{lem:self_normal}, we have 
\begin{align*}
\sum_{t\in \cT_j,t\geq T_0}   \|Z_{t,j}\|_{(V^{(t)}_j)^{-1}}
&\leq 2 \log \bigg(\frac{\det (V_j^{(T)})}{\det (V_j^{(T_0)})}\bigg) \\
&\leq 2d \log\bigg( \frac{\tr(V) +|\bar\cT_j| }{d} \bigg) - 2\log \det V_j^{(T_0)} ,
\end{align*}
where $\tr(V)\leq \sum_{i=1}^{T_0}\tr(Z_{i,j}Z_{i,j}^\top )\leq T_0$, and $\det (V_j^{(T_0+1)})\geq 1$ since 
$\lambda_{\min}(V_{j}^{(T_0+1)})\geq   \lambda_{\min}(V_j^{(T_0)})\geq 1$. Therefore, by the Cauchy-Schwarz inequality, we have 
\[
\sum_{t\in \cT_j,t\geq T_0}   \|Z_{t,j}\|_{(V^{(t)}_j)^{-1}} 
\leq \bigg( |\bar\cT_j| \sum_{t\in \cT_j,t\geq T_0}   \| Z_{t,j} \|_{(V_j^{(t)})^{-1}} ^2 \bigg)^{1/2}
\leq \sqrt{|\bar\cT_j|} \cdot \sqrt{2d\log(T/d)} 
\]
simultaneously for all $s\in[K]$
on the events in Proposition~\ref{prop:glm_err-n-choose-k}. 
Since $|\bar\cT_j|\leq |\cT_j|$, this further implies 
\$
R_T &\leq R_{T_0 } + 2\xi\cdot M_1 \cdot \bar{c}\sum_{j=1}^N \sqrt{|\cT_j|} \cdot \sqrt{2d\log(T/d)} \\
&\leq R_{T_0 } + 2\xi\cdot M_1 \cdot \bar{c} \cdot \sqrt{N} \cdot \sqrt{\sum_{j=1}^N  |\cT_j|} \cdot \sqrt{2d\log(T/d)} \\
&\leq R_{T_0 } + 2\xi\cdot M_1 \cdot \bar{c} \cdot \sqrt{NKT} \cdot   \sqrt{2d\log(T/d)} \\
&\leq R_{T_0} +\frac{2\sqrt{3} \bar\sigma }{\kappa } \cdot M_1\cdot \bar{c} \cdot \sqrt{NKT} \cdot \sqrt{ (d+1) \log(1+2T/d) +   \log(2/\delta)} \cdot  \sqrt{ 2d \log(T/d)} \\ 
&\leq R_0 T_0 + \frac{5\sigma}{\kappa}\cdot  M_1\cdot \bar{c} \cdot d\sqrt{NK T } \log(T/(d\delta))
\$
with probability at least $1-\delta$. 
Above, the second line uses the Cauchy-Schwarz inequality, and the third line uses the fact that $\sum_{i=1}^N |\cT_j| = \sum_{i=1}^N \sum_{t=1}^T \ind\{j\in s_t(X_t)\} = \sum_{t=1}^T|s_t(X_t)| = KT$. We thus conclude the proof.
\end{proof}

\subsection{Proof of Lemma~\ref{lemma:eigenvaluebound-nchoosek}}
\label{sec:eigenvaluebound-nchoosek}

\begin{proof}
By the definition of $V_j^{(t)}$, for any $t\geq T_0$, we have
\[
V_j^{(t)}=V_j^{(T_0)}+\sum_{i=T_0}^{t-1}z_j^{(i)}(z_j^{(i)})^{\top}\succeq V_j^{(T_0)}.
\]
Hence $\lambda_{\min}(V_j^{(t)})\geq\lambda_{\min}(V_j^{(T_0)})$. It suffices to bound $\lambda_{\min}(V_j^{T_0})$ simultaneously for all $j\in[N]$. Consider the distribution of
\[
V_j^{(T_0)}=\sum_{t=1}^{T_0-1}\mathbbm{1}\{j\in(X_t)\}Z_{t,j}(Z_{t,j})^{\top}.
\]
Due to the random sampling of $s_t(X_t)$ for all $t<T_0$, the matrix $V_j^{T_0}$ are summations of i.i.d. matrices with expectation
\begin{align*}
\Sigma_j:=\mathbb{E}[\mathbbm{1}\{j\in(X_t)\}Z_{t,j}(Z_{t,j})^{\top}]=\frac{K}{N}\mathbb{E}[Z_{t,j}(Z_{t,j})^{\top}],
\end{align*}
where
\begin{align*}
\mathbb{E}[Z_{t,j}(Z_{t,j})^{\top}]=\begin{pmatrix}
\mathbb{E}[\text{Unif}(-\frac{1}{2}+\frac{1}{K},-\frac{1}{2}+\frac{2}{K},\cdots,\frac{1}{2})]&0\\
0&\mathbb{E}[X_tX_t^{\top}]
\end{pmatrix}=\begin{pmatrix}
\frac{1}{2K}&0\\
0&\Sigma_{x,j}
\end{pmatrix}.
\end{align*}
Note that $(Z_{t,j})=(\sigma_t(j),X_t)$, and since $\sigma_{t}(j)$'s are rescaled, they follow Unif$(-\frac{1}{2}+\frac{1}{K},-\frac{1}{2}+\frac{2}{K},\cdots,\frac{1}{2})$. By the independence of $\sigma_t(j)$ and $X_t$, we derive the above result.

Hence $\lambda_{\min}(\Sigma_j)\geq\min\{\frac{1}{2K},\lambda_{\min}(\Sigma_{x,j})\}\geq c_1:=\min\{\frac{1}{2K},c_x\}$, where $\lambda_{\min}(\Sigma_{x,j})\geq c_x$. By Jensen's inequality, we have $\|\Sigma_j\|_{op}=\frac{K}{N}\|\mathbb{E}[Z_{t,j}(Z_{t,j})^{\top}]\|_{op}\leq\frac{K}{N}\mathbb{E}[\|Z_{t,j}\|^2]\leq\frac{K}{N}$. By the independence of each ranking before $T_0$, the random matrices $\{z_{t,j}(z_{t,j})^{\top}-\Sigma_{j}\}_{t=1}^{T_0-1}$ are i.i.d. centered in $\mathbb{R}^{(d+1)\times(d+1)}$ with uniformly bounded matrix operator norm:
\[
\|z_{t,j}(z_{t,j})^{\top}-\Sigma_{j}\|_{op}\leq\|z_{t,j}(z_{t,j})^{\top}\|+\|\Sigma_j\|_{op}=1+\frac{K}{N}=\frac{K+N}{N}.
\]

Now, define $\bar{V}_j:=V_j^{(T_0)}-T_0\Sigma_j=\sum_{t=1}^{T_0-1}\{z_{t,j}(z_{t,j})^{\top}-\Sigma_{j}\}$. By triangle inequality,
\begin{align*}
\|\mathbb{E}[\bar{V}_j\bar{V}_j^{\top}]\|_{op}=&\bigg\|\sum_{t=1}^{T_0-1}\mathbb{E}[\{z_{t,j}(z_{t,j})^{\top}-\Sigma_{j}\}\{z_{t,j}(z_{t,j})^{\top}-\Sigma_{j}\}^{\top}]\bigg\|_{op}\\
\leq&T_0\cdot\|\mathbb{E}[\{z_{t,j}(z_{t,j})^{\top}-\Sigma_{j}\}\{z_{t,j}(z_{t,j})^{\top}-\Sigma_{j}\}^{\top}]\|_{op}\\
\leq&T_0\cdot\mathbb{E}\|\{z_{t,j}(z_{t,j})^{\top}-\Sigma_{j}\}\{z_{t,j}(z_{t,j})^{\top}-\Sigma_{j}\}^{\top}\|_{op}\\
\leq&T_0\cdot\mathbb{E}\|z_{t,j}(z_{t,j})^{\top}-\Sigma_{j}\|_{op}\|z_{t,j}(z_{t,j})^{\top}-\Sigma_{j}\}^{\top}\|_{op}\leq T_0\bigg(\frac{K+N}{N}\bigg)^2.
\end{align*}
Similarly $\|\mathbb{E}[\bar{V}_j^{\top}\bar{V}_j]\|_{op}\leq T_0\frac{K+N}{N}$. Then by the matrix Bernstein's inequality, for any $t\geq0$:
\[
\mathbb{P}(\|\bar{V}_j\|_{op}\geq t)\leq2(d+1)\cdot\exp\bigg(-\frac{\frac{t^2}{2}}{T_0(\frac{K+N}{N})^2+\frac{2t}{3}}\bigg).
\]
The right-handed side is not greater than $\delta$ for any $t$ such that $t\geq\sqrt{2T_0(\frac{K+N}{N})^2\log\frac{2(d+1)}{\delta}}$ and $t\geq\frac{4}{3}\log\frac{2(d+1)}{\delta}$. Thus, with probability at least $1-\delta$, we have
\[
\|\bar{V}_j^{(T_0)}-T_0\Sigma_j\|_{op}\leq \sqrt{2T_0\bigg(\frac{K+N}{N}\bigg)^2\log\frac{2(d+1)}{\delta}}+\frac{4}{3}\log\frac{2(d+1)}{\delta},
\]
and hence
\[
\lambda_{\min}(V_j^{T_0})\geq T_0\cdot\lambda_{\min}(\Sigma_j)-\sqrt{2T_0\bigg(\frac{K+N}{N}\bigg)^2\log\frac{2(d+1)}{\delta}}-\frac{4}{3}\log\frac{2(d+1)}{\delta}.
\]
This implies that
\[
\lambda_{\min}(V_j^{T_0})\geq T_0/2\cdot\lambda_{\min}(\Sigma_j),
\]
as long as $\sqrt{2T_0(\frac{K+N}{N})^2\log\frac{2(d+1)}{\delta}}\leq T_0/4\cdot\lambda_{\min}(\Sigma_j)$ and $\frac{4}{3}\log\frac{2(d+1)}{\delta}\leq T_0/4\cdot\lambda_{\min}(\Sigma_j)$. Setting $T_0\geq(\frac{16}{3c_1}+\frac{32(K+N)^2}{N^2c_1})\log\frac{2(d+1)}{\delta}$, we have
\[
\lambda_{\min}(V_j^{T_0})\geq T_0c_1/2
\]
with probability at least $1-\delta$. Further let $T_0\geq\max\{(\frac{16}{3c_1}+\frac{32(K+N)^2}{N^2c_1})\log\frac{2(d+1)}{\delta},\frac{2B}{c_1}\}$, we have $\lambda_{\min}(V_j^{(T_0)})\geq B$ with probability at least $1-\delta$.
\end{proof}

\subsection{Proof of Proposition~\ref{prop:glm_err-n-choose-k}}

\label{app:prop_general}

\begin{proof}[Proof of Proposition~\ref{prop:glm_err-n-choose-k}]
Fix any $j\in [N]$. Throughout the proof, we condition on $\{X_i\}_{i=1}^T$, and define the martingale 
$\{\cH_{\tau,j}\}_{\tau=0}^{T}$, where 
\$
\cH_{\tau,j} = \sigma\big( \{Z_{1,i},Y_{1,i}, \cdots, Z_{\tau-1,i}, Y_{\tau-1,i}, Z_{\tau,i}\}_{i=1}^n\big)
\$
is the history of all rewards $Y_{t,i}$ for items that appeared in the recommended lists up to time $\tau-1$, and the ranking decisions $Z_{\tau,i}$ up to time $\tau$. 
We slightly deviate from the notation in the main text and use $Y_{t,i}$ for ease of presentation. 
Here without loss of generality, we impose $Y_{t,i}=0$ and $Z_{t,i}=0$ if $i\notin s_t(X_t)$. For any $t\geq 1$, we define 
\$
\cT_{t}^{(j)} = \{ \tau \leq t-1\colon j\in s(X_\tau)\}
\$
be the set of time steps where item $j$ appears in the recommended list.

For any $t\in [T]$ and $j\in[N]$, we define
\$
L_{t,j}(\theta) =  
\sum_{\tau \in \mathcal{T}_t^{(j)} } Y_{\tau,j} \theta^\top Z_{\tau,j} - A(\theta^\top Z_{\tau,j}) ,\quad 
\nabla L_{t,j}(\theta) = \sum_{\tau \in \mathcal{T}_t^{(j)} } Y_{\tau,j}  Z_{\tau,j} - A'(\theta^\top Z_{\tau,j})Z_{\tau,j}.
\$
By the first-order condition of the MLE for $\hat\theta_{t,j}$, we have $\nabla L_{t,j}(\hat\theta_{t,j})=0$, and 
\$
\nabla L_{t,j}(\theta_j) = \sum_{\tau \in \mathcal{T}_t^{(j)} } Z_{\tau,j}\big(Y_{\tau,j}   - A'(\theta_j^\top Z_{\tau,j})\big):= \sum_{\tau \in \mathcal{T}_t^{(j)} } Z_{\tau,j}\epsilon_{\tau,j},
\$
where $\epsilon_{\tau,j}:=Y_{\tau,j}   - A'(\theta_j^\top Z_{\tau,j})$ obeys 
$\EE[\epsilon_{\tau,j}\given \cH_{\tau,j}] = 0$ due to our models. 
By the mean value theorem, there exists 
some $\tilde\theta_{t,j}$ 
that lies on the segment between $\theta_j$ and $\hat\theta_{t,j}$, such that 
\$
\sum_{\tau=1}^{t-1} Z_{\tau,j}\epsilon_{\tau,j}
= \nabla L_{t,j}(\theta_j) - \nabla L_{t,j}(\hat\theta_{t,j})
= \nabla^2 L_{t,j}(\tilde\theta_{t,k})(\theta_j - \hat\theta_{t,j}),
\$ 
where $\nabla^2 L_{\tau,j}(\theta)$ 
is the Hessian matrix of $L_{\tau,j}$ at $\theta$. That is, 
\$
\sum_{\tau=1}^{t-1} Z_{\tau,j}\epsilon_{\tau,j}
= \sum_{\tau=1}^{t-1} A''(\tilde\theta_{t,j}^\top Z_{\tau,j})Z_{\tau,j}Z_{\tau,j}^\top (\theta_j - \hat\theta_{t,j}).
\$
Recalling  (ii) $\kappa := \inf_{\|z\|\leq 1, \|\theta-\theta_k\|\leq 1}  {A''}(\theta^\top z) >0$ 
in Assumption~\ref{assump:regularity}(b), 
and noting that $V_j^{(t)}=\sum_{\tau=1}^{t-1} Z_{\tau,j}Z_{\tau,j}^\top$, 
we know that 
\begin{align}\label{eq:bd1-nchoosek}
\Big\|\sum_{\tau=1}^{t-1} Z_{\tau,j}\epsilon_{\tau,j}\Big\|_{(V_j^{(t)})^{-1}}^2 \geq \kappa^2 \|\hat\theta_{t,j}-\theta_k\|_{V_j^{(t)}}^2
\end{align}
holds as long as $\|\hat\theta_{t,j}-\theta_j\|_2\leq 1$. 

In the sequel, we are to show that 
$\|\hat\theta_{t,j}-\theta_j\|_2\leq 1$ with high probability. 
Let 
\$
S_{t,j}(\theta) = \nabla L_{t,j}(\theta_j) - \nabla L_{t,j}(\theta) 
= \sum_{\tau \in \mathcal{T}_t^{(j)} } \big( A'(\theta_j^\top Z_{\tau,j}) - A'(\theta^\top Z_{\tau,j})\big) Z_{\tau,j}.
\$
Note that $A'$ is increasing, hence $A''(\cdot)>0$. Also, as long as $V_j^{(t)}>0$, we know that $S_{t,j}(\theta)$ is strictly convex in $\theta\in \RR^{d+1}$, and thus $S_{t,j}$ is an injection from $\RR^{d+1}$ 
to $\RR^{d+1}$. 
Furthermore, for any $\theta$ with $\|\theta-\theta_j\|\leq 1$,  
there exists some $\tilde\theta$ that lies on the segment between $\theta$ and $\theta_j$ such that
\$
\|S_{t,j}(\theta) \|_{(V_j^{(t)})^{-1}}^2 
= \Big\| \sum_{\tau=1}^{t-1} A''(\tilde\theta^\top Z_{\tau,j})  Z_{\tau,j}Z_{\tau,j}^\top (\theta_j -\theta)\Big\|_{(V_j^{(t)})^{-1}}^2 
\geq \kappa^2 \lambda_{\min}(V_j^{(t)}) \|\theta_j-\theta\|^2.
\$
The above arguments verify the conditions needed in \citet[Lemma~A]{chen1999strong}; it then implies  
for any $r>0$, 
\#\label{eq:lemA_chen_gen-nchoosek}
\Big\{\theta\colon \|S_{t,j}(\theta) \|_{(V_j^{(t)})^{-1}}^2 \leq \kappa ^2 r^2 {\lambda_{\min}(V_j^{(t)})} \Big\} \subseteq \{\theta\colon \|\theta-\theta_j\|\leq r\}.
\#
Note that~\eqref{eq:lemA_chen_gen-nchoosek} is a deterministic result. 
Recall that $S_{t,j}(\hat\theta_{t,j}) = \sum_{\tau \in \mathcal{T}_t^{(j)} } Z_{\tau,j}\epsilon_{\tau,j}$, 
and $\epsilon_{\tau,j}$ are mean-zero conditional on $\cH_{\tau,j}$. We now use a more coarse martingale 
\$
\cH_{i,j}^{0} = \cH_{\tau_{i,j},j},
\$
where $\tau_{i,j}$ is the $i$-th time that $j$ appears in $s_t(X_t)$; put differently, $\tau_{i,j}$ is the $i$-th smallest member in $\mathcal{T}_T^{(j)}$. 
We know that for any $\tau=\tau_{i,j}$ for any $i\leq |\cT_t^{(j)}|$, we still have $\EE[\epsilon_{\tau,j} \given \cH_{i,j}^0] =0$ due to the fact that $\EE[\epsilon_{\tau,j}\given \cH_{\tau,j}] = 0$  we discussed before.  
Therefore, we can invoke the concentration inequality of self-normalized processes in Lemma~\ref{lem:self_normalize_t}
for 
\$\sum_{\tau \in \mathcal{T}_t^{(j)} } Z_{\tau,j}\epsilon_{\tau,j}
= \sum_{i=1}^{|\mathcal{T}_t^{(j)}|} Z_{\tau_{i,j},j} \epsilon_{\tau_{i,j},j} \quad \textrm{and}\quad 
\bar{V}_{t,j} := \lambda I + \sum_{\tau \in \mathcal{T}_t^{(j)} } Z_{\tau,j}Z_{\tau,j}^\top 
= \lambda I + \sum_{i=1}^{|\mathcal{T}_t^{(j)}|} Z_{\tau_{i,j},j}Z_{\tau_{i,j},j}^\top 
\$ 
for some fixed $\lambda>0$, which yields that 
with probability at least $1-\delta$, 
it holds for all $t\geq 1$ that 
\$
\Big\|\sum_{\tau \in \mathcal{T}_t^{(j)} } Z_{\tau,j}\epsilon_{\tau,j}\Big\|_{\bar{V}_{t,j}^{-1}}^2 \leq 2 \bar\sigma^2 \log\Bigg( \frac{\det (\bar{V}_{t,j})^{1/2} \det(\lambda I)^{-1/2}}{\delta} \Bigg).
\$
Note that $\det(\lambda I) = \lambda^{d+1}$, and $\det(\bar{V}_{t,j})\leq (\lambda_{\max}(\bar{V}_{t,j}))\leq (\lambda + |\mathcal{T}_t^{(j)}|/d)^{d+1}$ since $\|Z_{\tau,j}\|_2\leq 1$ by Lemma~\ref{lem:determ}. We then have 
\#\label{eq:bd_barV_gen-nchoosek}
\Big\|\sum_{\tau \in \mathcal{T}_t^{(j)} } Z_{\tau,j}\epsilon_{\tau,j}\Big\|_{\bar{V}_{t,j}^{-1}}^2 \leq 2 \bar\sigma^2 \big( (d+1) \log(1+|\mathcal{T}_t^{(j)}|/(d\lambda)) + \log(1/\delta)\big) 
\#
with probability at least $1-\delta$ for all $t\geq 1$. 

By Lemma~\ref{lemma:eigenvaluebound-nchoosek}, we know that
\#\label{eq:lower_Vj_eigen-nchoosek}
\lambda_{\min}(V_j^{(t)})\geq \frac{3 \bar\sigma^2 }{\kappa^2}\big( (d+1) \log(1+2t/d) + \log(1/\delta)\big)
\#
holds with probability at least $1-\delta$ since 
$
T_0 \geq \max\big\{  (\frac{16}{3c_1}+\frac{32(K+N)^2}{N^2c_1})\log\frac{2(d+1)}{\delta} , ~\frac{6\bar\sigma^2}{c_1\kappa^2} ( (d+1) \log(1+2t/d) + \log(1/\delta) )  \big\}.
$ 
Finally, since $\lambda_{\min}(V_j^{(t)})\geq \lambda_{\min}(V_j^{(T_0)}) \geq 1$ 
for $t\geq T_0$,  take $\lambda=1/2$ and note that 
\#\label{eq:bd_barV_hatV_gen-nchoosek}
\bar{V}_{t,j} = \lambda I + V_j^{(t)} 
\preceq (1+\lambda) V_j^{(t)}
 = 3/2\cdot V_j^{(t)}.
\#
By definition, we have 
$\|S_{t,j}(\hat\theta_{t,j})\|_{(V_j^{(t)})^{-1}}^2  = \big\|\sum_{\tau \in \mathcal{T}_t^{(j)} } Z_{\tau,j}\epsilon_{\tau,j}\big\|_{(V_j^{(t)})^{-1}}^2$. Thus,
combining~\eqref{eq:bd_barV_gen-nchoosek} and~\eqref{eq:bd_barV_hatV_gen-nchoosek},
we have 
\#\label{eq:bd_hatV_gen-nchoosek}
\|S_{t,j}(\hat\theta_{t,j})\|_{(V_j^{(t)})^{-1}}^2 
&\leq 3 \bar\sigma^2 \big( (d+1) \log(1+2|\mathcal{T}_t^{(j)}|/d) + \log(1/\delta)\big) \notag \\ 
&\leq 3 \bar\sigma^2 \big( (d+1) \log(1+2t/d) + \log(1/\delta)\big)
\#
for all $t\geq T_0$ with probability at least $1-\delta$. 
Taking a union bound on the above two facts, 
we know that 
\#\label{eq:event_consistent-nchoosek}
{\lambda_{\min}(V_j^{(t)})} \geq \frac{1}{\kappa^2 }\|S_{t,j}(\theta) \|_{(V_j^{(t)})^{-1}}^2 
\#
holds
 with probability at least $1-2\delta$.
Now we apply Lemma~A of \cite{chen1999strong} again, which implies that if we set $r=1$, then $\|\hat\theta_{t,j}-\theta_j\|\leq 1$ on the 
event~\eqref{eq:event_consistent-nchoosek}, which further implies that  $\|\hat\theta_{t,j}-\theta_k\|\leq 1$ for all $t\geq T_0$ with probability at least $1-2\delta$. Applying  \eqref{eq:bd1-nchoosek}, we know that with probability at least $1-2\delta$, 
\[
\|\hat\theta_{t,j}-\theta_j\|_{V_j^{(t)}}^2 
\leq \frac{1}{\kappa^2}
\Big\|\sum_{\tau \in \mathcal{T}_t^{(j)} } Z_{\tau,j}\epsilon_{\tau,j}\Big\|_{(V_j^{(t)})^{-1}}^2 \leq \frac{3 \bar\sigma^2}{\kappa^2} \big( (d+1) \log(1+2t/d) + \log(1/\delta)\big).
\]
Therefore, we conclude the proof by replacing $\delta$ with $\delta/2$.
\end{proof}

\subsection{Proof of Lemma~\ref{lem:regret_general-nchoosek}}
\label{app:regret_lemma}
\begin{proof}[Proof of Lemma~\ref{lem:regret_general-nchoosek}]
    
Denote $\xi = \frac{\sqrt{3} \bar\sigma }{\kappa } \sqrt{ (d+1) \log(1+2T/d) + \log(2/\delta)}$. 
From Proposition~\ref{prop:glm_err-n-choose-k}, 
we know that with probability at least $1-\delta$,
\begin{equation}\label{eq:ucb_event_mu_general-nchoosek}
\hat{\mu}_{t,q_t(k)}^L(x,\sigma(k))\leq  \mu_{t,q_t(k)}(x,\sigma(k)) \leq \hat{\mu}^U_{t,q_t(k)}(x,\sigma(k))
\end{equation}
holds  for  all $x\in\cX,\sigma\in S_K,k\in[K], t\in[T_0,T]$, where 
\[
\hat{\mu}^U_{t,j}(z) := A'( z^\top \hat\theta_{t,j} + \xi \cdot \|z\|_{(V_j^{(t)})^{-1}} ),\quad 
\hat{\mu}^L_{t,j}(z) := A'( z^\top \hat\theta_{t,j} - \xi \cdot \|z\|_{(V_j^{(t)})^{-1}} ).
\]
For any $x\in \cX$ and any $\sigma\in S_K$, 
for notational convenience, we denote
\begin{align*}
&\hat{U}_{t}(x,\sigma) = H\big(g_k(\hat{\mu}^U_{t,q_t(1)}(x, \sigma(1)))+ \dots + g_k(\hat{\mu}^U_{t,q_t(K)}(x, \sigma(K)))\big), \\
&\hat{L}_{t}(x,\sigma) = H\big(g_k(\hat{\mu}^L_{t,q_t(1)}(x, \sigma(1)))+ \dots + g_k(\hat{\mu}^L_{t,q_t(K)}(x, \sigma(K)))\big).
\end{align*}
Then, since $H$ and $g_k$ are monotone, we know that $\hat{L}_t(x,\sigma)$ and $\hat{U}_t(x,\sigma)$ are valid LCBs and UCBs:
\begin{equation}\label{eq:ucb_event_r_general-nchoosek}
\PP\big(\hat{L}_{t}(x,\sigma)\leq r(x,\sigma) \leq \hat{U}_{t}(x,\sigma),~\forall x\in \cX,\sigma\in S_K,t\in [T_0,T] \big) \geq 1-\delta.
\end{equation}

By definition, the regret at $T\geq T_0$ can be upper bounded as 
\begin{align*}
R_T &\leq R_{T_0} + \sum_{t=T_0}^T \big\{\hat{U}_{t}(X_t,\sigma_t)  
- \hat{L}_{t}(X_t, \sigma_t)   \big\} \\ 
& = R_{T_0} + \sum_{t=T_0}^T  H\Big(\sum_{k=1}^K g_k\big(\hat{\mu}^U_{t,q_t(k)}(X_t, \sigma_t(k)\big)\Big) -  H\Big(\sum_{k=1}^K g_k\big(\hat{\mu}^L_{t,q_t(k)}(X_t, \sigma_t(k)\big)\Big) \\ 
&\leq R_{T_0} + \sum_{t=T_0}^T \sum_{k=1}^K c_k
\Big[A'( Z_{t,q_t(k)}^\top \hat\theta_{t,q_t(k)} + \xi \cdot \|Z_{t,q_t(k)}\|_{(V^{(t)}_{q_t(k)})^{-1}} ) \\ 
&\qquad \qquad \qquad \qquad \qquad \qquad -A'( Z_{t,q_t(k)}^\top \hat\theta_{t,q_t(k)} - \xi \cdot \|Z_{t,q_t(k)}\|_{(V^{(t)}_{q_t(k)})^{-1}} ) \Big] \\
&\leq R_{T_0} + M_1\cdot 2\xi \cdot \sum_{t=T_0}^T \sum_{k=1}^K c_k  \cdot \|Z_{t,q_t(k)}\|_{(V^{(t)}_{q_t(k)})^{-1}} ,
\end{align*}
where the last inequality follows from the fact that the first derivative of $A'$ is upper bounded by $M_1$.  We thus conclude the proof of Lemma~\ref{lem:regret_general-nchoosek}.
\end{proof}

\section{Technical Proofs for $n=K$}

\subsection{Proof of Theorem~\ref{thm:n=k}}

\begin{proof}[Proof of Theorem~\ref{thm:n=k}]
When $n=K$, without loss of generality 
we have $s(X)\equiv \{1,2,\dots,K\}$ 
for a fixed ordering of $K$ items. That is, $q_t(k)=k$ 
for any $t\in[T]$ and any $k\in[K]$. 
For any $z=(i,x)$, $i\in [K]$ and $x\in\cX$, and any 
$t\in[T]$, $k\in[K]$, recall that 
\$
\muhatupp{t}{k}(z) := A'( z^\top \hat\theta_{t,k} + \xi \cdot \|z\|_{(\covmat{t}{k})^{-1}} ).
\$
We also define the lower confidence bound as
\$
\muhatlow{t}{k}(z) := A'( z^\top \hat\theta_{t,k} - \xi \cdot \|z\|_{(\covmat{t}{k})^{-1}} ).
\$
For any $x\in \cX$ and any $\sigma\in S_K$, 
for notational convenience, we denote
\$
&\rupp{t}(x,\sigma) = H\big(\muhatupp{t}{1}(x, \sigma(1)), \dots, \muhatupp{t}{K}(x, \sigma(K))\big), \\
&\rlow{t}(x,\sigma) = H\big(\muhatlow{t}{1}(x, \sigma(1)), \dots, \muhatlow{t}{K}(x, \sigma(K))\big).
\$ 
We will prove that with probability at least $1-\delta$,
\#\label{eq:ucb_event_mu}
\muhatlow{t}{k}(x,\sigma(k))\leq  \mutrue{t}{k}(x,\sigma(k)) \leq \muhatupp{t}{k}(x,\sigma(k)),~\text{for all }x\in\cX,\sigma\in S_K,k\in[K], t\in[T],  
\# 
which further implies the UCB conditions on the true rewards $\rtrue(x,\sigma)$:
\#\label{eq:ucb_event_r}
\PP\Big( \rlow{t}(x,\sigma)\leq \rtrue(x,\sigma) \leq \rupp{t}(x,\sigma) \Big) \geq 1-\delta.
\#

Recall that 
\$  
    \{\hat{\alpha}_{t,k}, \hat{\beta}_{t,k}\} =\argmax_{\alpha, \beta}\sum_{\tau=1}^t Y_{\tau,k}\big(\alpha \sigma_{\tau}(k) + \beta^T X_t\big) - A\big(\alpha\sigma_{\tau}(k) + \beta^T X_t\big).
\$
is the MLE of $\alpha_k,\beta_k$ using data up to time $t$. 
We are to prove the consistency of $\hat\alpha_{t,k}$ and 
$\hat\beta_{t,k}$ and leverage it to construct valid 
UCBs. 
To simplify notations, 
for any $t\in[T]$, we denote the augmented feature 
and parameters  as 
\$
Z_{t,k} = (\sigma_t(k), X_t),\quad \theta_k = (\alpha_k,\beta_k), \quad \text{and} \quad \hat\theta_{t,k} = (\hat\alpha_{t,k}, \hat\beta_{t,k}).
\$
Then, our point estimates can be written as 
\$
\hat\theta_{t,k} =\argmax_{\theta}\sum_{\tau=1}^t Y_{\tau,k}\theta^\top Z_{\tau,k} - A\big(\theta^\top Z_{\tau,k}\big).
\$
We also define the empirical covariance matrices as 
$\covmat{t}{k} := \sum_{\tau=1}^{t-1} Z_{k,t}Z_{k,t}^\top$.

\begin{lemma}[Eigenvalue of $\covmat{t}{k}$]\label{lem:eigen_Vk}
Fix any $\delta\in(0,1)$, and let 
$c_1   := \min\{\frac{1}{12}+\frac{1}{6K^2},c_x\}>0$. 
Let $B>0$ be any positive constant. Fix any $k\in[K]$. Suppose
\$
T_0 \geq \max\bigg\{ \bigg(\frac{32}{3c_1}+\frac{256}{c_1^2} \bigg) \log\bigg(\frac{2d+2}{\delta}\bigg)  , ~\frac{2 B}{c_1} \bigg\},  
\$
Then with probability at least $1-\delta$, 
it holds that $\lambda_{\min}(V_k^{(t)})\geq B$ for all $t\geq T_0$.
\end{lemma}

\begin{prop}[Estimation error of $\hat\theta_{t,k}$]
\label{prop:glm_err_theta}
Fix any $\delta\in(0,1/4)$, and suppose 
\#\label{eq:T0_bd}
T_0 \geq \max\bigg\{ \bigg(\frac{32}{3c_1}+\frac{256}{c_1^2} \bigg) \log\bigg(\frac{4d+4}{\delta}\bigg) ,~\frac{6\bar\sigma^2}{c_1\cdot\kappa^2}\big((d+1) \log(1+2T/d) + \log(2/\delta)\big) \bigg\}.
\# 
Then, with probability at least 
$1-\delta$, it holds simultaneously for all $T_0\leq t\leq T$ and  $k\in[K]$ that 
\#\label{eq:glm_err}
\big\|  \hat\theta_{t,k}-\theta_k \big\|_{\covmat{t}{k}} \leq  \frac{\sqrt{3} \bar\sigma }{\kappa } \sqrt{ (d+1) \log(1+2t/d) + \log(2/\delta)} .
\# 
\end{prop}

Now let $\xi = \frac{\sqrt{3} \sigma }{\kappa } \sqrt{ (d+1) \log(1+2T/d) + \log(2/\delta)}$. By 
Proposition~\ref{prop:glm_err_theta}, we know that 
if $T_0$ is sufficiently great that it obeys~\eqref{eq:T0_bd}, 
then by Cauchy-Schwarz inequality, it holds that 
\$
z^\top (\hat\theta_{t,k} - \theta_k)
\leq \|z\|_{(\covmat{t}{k})^{-1}} \cdot \|\hat\theta_{t,k} - \theta_k\|_{\covmat{t}{k}} \leq \xi \cdot \|\hat\theta_{t,k} - \theta_k\|_{\covmat{t}{k}},\quad \forall \|z\|\leq 1, ~ t\geq T_0
\$
with probability at least $1- \delta$. 
Thus, 
the event~\eqref{eq:ucb_event_mu} holds with probability at least $1- \delta$, 
hence~\eqref{eq:ucb_event_r} holds. 
On the event in~\eqref{eq:ucb_event_r}, the regret can be bounded as 
\$
R_T &= R_{T_0} + \sum_{t=T_0+1}^T \rtrue(X_t,\sigma_t^*) - \rtrue(X_t,\sigma_t) \\
&=  R_{T_0}+ \sum_{t=T_0}^T
\big\{ \rupp{t}(X_t,  \sigma_t^*) 
+  \rtrue(X_t,\sigma_t^*) - \rupp{t}(X_t,\sigma_t^*)
- \rupp{t}(X_t, \sigma_t) 
+ \rupp{t}(X_t, \sigma_t) 
- \rtrue(X_t, \sigma_t) \big\} \notag \\ 
&\leq R_{T_0}+ \sum_{t=T_0}^T
\big\{   \rtrue(X_t,\sigma_t^*) -\rupp{t}(X_t,\sigma_t^*)
+ \rupp{t}(X_t, \sigma_t)  
- \rtrue(X_t,  \sigma_t) \big\} \notag \\ 
&\leq R_{T_0}+ \sum_{t=T_0}^T
\big\{   \rupp{t}(X_t, \sigma_t)  
- \rtrue(X_t, \sigma_t) \big\}  
 \leq R_{T_0} + \sum_{t=T_0}^T \big\{   \rupp{t}(X_t,\sigma_t)  
- \rlow{t}(X_t, \sigma_t)   \big\},
\$
where the second inequality 
uses the fact that 
$\rupp{t}(X_t,\sigma_t) \geq \rupp{t}(X_t,\sigma_t^*)$, 
and the third and fourth inequalities 
follow from the event~\eqref{eq:ucb_event_r}.  

By the monotonicity and Lipschitz  conditions in Assumption~\ref{assump:regularity}(a), we know that 
\$
0 &\leq \rupp{t}(X_t, \sigma_t)  
- \rlow{t}(X_t, \sigma_t)   \\
&=  H\bigg(\sum_{k=1}^Kg_k(\muhatupp{t}{k}(X_t, \sigma_t(k)))\bigg)
- H\bigg(\sum_{k=1}^Kg_k(\muhatlow{t}{k}(X_t, \sigma_t(k)))\bigg) \\ 
&\leq \sum_{s=1}^{K} \Bigg\{ H\bigg(\sum_{k=1}^s g_k(\muhatupp{t}{k}(X_t, \sigma_t(k))) 
+ \sum_{k=s+1}^Kg_k(\muhatlow{t}{k}(X_t, \sigma_t(k))) \bigg) \\ 
&\qquad \qquad- H\bigg(\sum_{k=1}^{s-1}g_k(\muhatupp{t}{k}(X_t, \sigma_t(k))) + \sum_{k=s}^K g_k(\muhatlow{t}{k}(X_t, \sigma_t(k)))\bigg) \Bigg\} \\ 
&\leq \sum_{s=1}^K c_s \cdot \big[  \muhatupp{t}{s}(X_t, \sigma_t(s)) - \muhatlow{t}{s}(X_t, \sigma_t(s)) \big] \\ 
&=  \sum_{k=1}^K c_k
\Big[A'( Z_{t,k}^\top \hat\theta_{t,k} + \xi \cdot \|Z_{t,k}\|_{(\covmat{t}{k})^{-1}} ) -A'( Z_{t,k}^\top \hat\theta_{t,k} - \xi \cdot \|Z_{t,k}\|_{(\covmat{t}{k})^{-1}} ) \Big] \\ 
&\leq M_1\cdot \sum_{k=1}^K c_{k} \cdot 2 \xi \cdot \|Z_{t,k}\|_{(\covmat{t}{k})^{-1}} ,
\$ 
where $Z_{t,k} = (X_t,\sigma_t(k))$ is the aggregated feature for item $k$ at time $t$. 
We thus have 
\$
R_T \leq 
R_{T_0} +2\xi\cdot M_1\cdot \sum_{k=1}^K c_{k}  \sum_{t=T_0}^T  \|Z_{t,k}\|_{(\covmat{t}{k})^{-1}}.
\$

We then use the self-normalized concentration inequality (c.f.~Lemma~\ref{lem:self_normal}) to bound the RHS above. 
Under the notations of Lemma~\ref{lem:self_normal}, 
fixing any $s\in [K]$, 
we set $X_t = Z_{t,s}$, 
$V= V_k^{(T_0)}$, and note that 
$\bar{V}_{t-T_0+1} := V_k^{(t)}=V + \sum_{i=T_0}^{t} X_tX_t^\top$. 
Also, on the event in Proposition~\ref{prop:glm_err_theta}
we see $\lambda_{\min}(V_s^{(T_0)})\geq 1$ 
while $\|Z_{t,s}\|\leq 1$. 
Therefore, invoking Lemma~\ref{lem:self_normal}, we have 
\$
\sum_{t=T_0}^t \| Z_{t,s} \|_{(V_k^{(t)})^{-1}}^2
&\leq 2 \log \bigg(\frac{\det (V_k^{(t)})}{\det (V_k^{(T_0)})}\bigg) \\
&\leq 2d \log\bigg( \frac{\tr(V) +t-T_0}{d} \bigg) - 2\log \det V_k^{(T_0)} ,
\$
where $\tr(V)\leq \sum_{i=1}^{T_0}\tr(Z_{i,s}Z_{i,s}^\top )\leq T_0$, and $\det (V_s^{(T_0+1)})\geq 1$ since 
$\lambda_{\min}(V_{s}^{(T_0+1)})\geq   \lambda_{\min}(V_s^{(T_0)})\geq 1$. Therefore, by the Cauchy-Schwarz inequality, we have 
\$
\sum_{t=T_0}^t \|Z_{t,s}\|_{(V_k^{(t)})^{-1}} 
\leq \bigg( (t-T_0) \sum_{t=T_0}^t \| Z_{t,s} \|_{(V_k^{(t)})^{-1}} ^2 \bigg)^{1/2}
\leq \sqrt{t-T_0} \cdot \sqrt{2d\log(T/d)} 
\$
simultaneously for all $s\in[K]$
on the events in Proposition~\ref{prop:glm_err_theta}. 
This further implies 
\$
R_T &\leq R_{T_0 } + 2\xi\cdot M_1 \cdot c_H \sqrt{T-T_0} \cdot \sqrt{2d\log(T/d)} \\
&\leq R_{T_0} +\frac{2\sqrt{3} \sigma }{\kappa } \cdot M_1\cdot c_H \sqrt{ T(d+1) \log(1+2T/d) + T \log(2/\delta)} \cdot  \sqrt{ 2d \log(T/d)} \\ 
&\leq R_0 T_0 + \frac{5\sigma}{\kappa}\cdot M_1\cdot c_H \cdot d\sqrt{ T } \log(T/(d\delta))
\$
with probability at least $1-\delta$, where we denote $c_H :=\sum_{k=1}^K c_{k}$.
\end{proof}

\subsection{Proof of Lemma~\ref{lem:eigen_Vk}}

\begin{proof}[Proof of Lemma~\ref{lem:eigen_Vk}]
By definition, for any $t\geq T_0$ we have 
$V_k^{(t)}= V_k^{(T_0)} +\sum_{i=T_0}^{t-1} \Zki(\Zki)^\top 
\succeq V_k^{(T_0)}$, hence $\lambda_{\min}(V_k^{(t)})\geq 
\lambda_{\min}(V_k^{(T_0)})$. It thus suffices to bound 
$\lambda_{\min}(V_k^{(T_0)})$ simultaneously for all $k\in [K]$. 
Due to random sampling, 
$
V_k^{(T_0)} = \sum_{i=1}^{T_0-1} \Zki {\Zki}^\top 
$
are summations of i.i.d.~matrices, where each item has covariance matrix 
$
\Sigma := \EE\big[ \Zkone (\Zkone)^\top   \big] .
$
Here $\Zkone= (\sigma_1(k),x_1)$, where after our rescaling 
$\sigma_1(k)\sim \textrm{Unif}(-1/2+1/K,-1/2+2/K,\dots, 1/2)$, 
and is independent of $x_1$ obeying $\EE[x_1]=0$ 
and $\EE[x_1x_1^\top]=\Sigma_x$. We then have 
\$
\Sigma = \begin{pmatrix}
\frac{1}{12}+\frac{1}{6K^2} & 0 \\ 
0 & \Sigma_x
\end{pmatrix},
\$
hence $\lambda_{\min}(\Sigma)=\min\{\frac{1}{12}+\frac{1}{6K^2},\lambda_{\min}(\Sigma_x)\} \geq c_1:= \min\{\frac{1}{12}+\frac{1}{6K^2},c_x\}>0$. 
Also, by Jensen's inequality we have 
$\|\Sigma\|_{\textrm{op}}=\|\EE[\Zkone(\Zkone)^\top]\|_{\textrm{op}} \leq \EE[\|\Zkone\|^2] \leq 1$.
Furthermore, $\{\Zki(\Zki)^\top -\Sigma\}_{i=1}^{T_0-1}$ 
are i.i.d.~centered random matrices in $\RR^{(d+1)\times (d+1)}$ 
with uniformly bounded  matrix operator norm, i.e., by the triangle inequality, 
\$
\big\|\Zki(\Zki)^\top -\Sigma\big\|_{\textrm{op}}\leq 
\big\|\Zki(\Zki)^\top\big\|_{\textrm{op}} + \|\Sigma \|_{\textrm{op}} \leq 2.
\$   
Let $A:=V_k^{(T_0)} - T_0\Sigma
= \sum_{i=1}^{T_0-1}\{\Zki(\Zki)^\top -\Sigma\} $, 
then by the triangle inequality, 
\$
\big\|\EE[AA^\top]\big\|_{\textrm{op}}
&= \Big\|
\sum_{i=1}^{T_0-1}  \EE \big[\{\Zki(\Zki)^\top -\Sigma\}
 \{\Zki(\Zki)^\top -\Sigma\}^\top \big]\Big\|_{\textrm{op}} \\
&\leq T_0 \cdot \Big\| \EE \big[\{\Zki(\Zki)^\top -\Sigma\}
 \{\Zki(\Zki)^\top -\Sigma\}^\top \big]\Big\|_{\textrm{op}} \\ 
&\leq T_0 \cdot  \EE  \big\| \{\Zki(\Zki)^\top -\Sigma\}
 \{\Zki(\Zki)^\top -\Sigma\}^\top \big\|_{\textrm{op}}\\
&\leq T_0 \cdot  \EE  \big\| \Zki(\Zki)^\top -\Sigma\big\|_{\textrm{op}}
\big\|\Zki(\Zki)^\top -\Sigma\}^\top \big\|_{\textrm{op}} \leq 4T_0.
\$
Similar computation yields $\big\|\EE[A^\top A]\big\|_{\textrm{op}} \leq 4T_0$. 
Then, invoking the matrix Bernstein's inequality (c.f.~Lemma~\ref{lem:matrix_concentration}), 
for all $t\geq 0$, we have 
\$
\PP\big( \|A\|_{\oper} \geq t\big) \leq 2(d+1)\cdot \exp\Big( -\frac{t^2/2}{4T_0 + 2t/3} \Big).
\$
The right-handed side is no greater than $\delta$ 
for any $t$ such that  $t\geq 4\sqrt{T_0\log(2(d+1)/\delta)}$ 
and $t\geq  8/3\cdot\log(2(d+1)/\delta ) $. 
Therefore, with probability at least $1-\delta$, we have 
\$
\big\|V_k^{(T_0)} - T_0\Sigma\big\|_{\oper} \leq 
4\sqrt{T_0\log(2(d+1)/\delta)} +  8/3\cdot\log(2(d+1)/\delta ),
\$
and hence 
\$
\lambda_{\min} (V_k^{(T_0)}) \geq T_0 \cdot \lambda_{\min}(\Sigma) 
- 4\sqrt{T_0\log(2(d+1)/\delta)} -  8/3\cdot\log(2(d+1)/\delta ).
\$
The above implies 
\$
\lambda_{\min}(V_k^{(T_0)})\geq   T_0/2\cdot \lambda_{\min}(\Sigma)
\$ 
as long as 
$4\sqrt{T_0\log(2(d+1)/\delta)}\leq T_0/4\cdot \lambda_{\min}(\Sigma)$ and $8/3\cdot\log(2(d+1)/\delta )\leq  T_0/4\cdot \lambda_{\min}(\Sigma)$. 
Therefore, supposing 
$T_0 \geq (\frac{32}{3c_1}+\frac{256}{c_1^2})\log(\frac{2d+2}{\delta}) $, 
we have 
\$\lambda_{\min} (V_k^{(T_0)}) \geq T_0 c_1/2\$ with 
probability at least $1-\delta$. 
Further letting $T_0\geq 2B/c_1$, e.g., by letting 
$T_0\geq \max\{(\frac{32}{3c_1}+\frac{256}{c_1^2})\log(\frac{2d+2}{\delta}) 
 , 2B/c_1\}$, we have $\lambda_{\min} (V_k^{(T_0)})\geq B$ 
 with probability at least $1-\delta$. 
We thus conclude the proof of Lemma~\ref{lem:eigen_Vk}.
\end{proof}

\subsection{Proof of Proposition~\ref{prop:glm_err_theta}}

\begin{proof}[Proof of Proposition~\ref{prop:glm_err_theta}]
For any $t\in [T]$ and $k\in[K]$, we define
\$
L_{t,k}(\theta) =  
\sum_{\tau=1}^{t-1} Y_{\tau,k} \theta^\top Z_{\tau,k} - A(\theta^\top Z_{\tau,k}) ,\quad 
\nabla L_{t,k}(\theta) = \sum_{\tau=1}^{t-1} Y_{\tau,k}  Z_{\tau,k} - A'(\theta^\top Z_{\tau,k})Z_{\tau,k} 
\$
By the first-order condition of the MLE for $\hat\theta_{t,k}$, we have $\nabla L_{t,k}(\hat\theta_{t,k})=0$, and 
\$
\nabla L_{t,k}(\theta_k) = \sum_{\tau=1}^{t-1} Z_{\tau,k}\big(Y_{\tau,k}   - A'(\theta_k^\top Z_{\tau,k})\big):= \sum_{\tau=1}^{t-1} Z_{\tau,k}\epsilon_{\tau,k},
\$
where $\epsilon_{\tau,k}:=Y_{\tau,k}   - A'(\theta_k^\top Z_{\tau,k})$ obeys 
$\EE[\epsilon_{\tau,k}\given \cH_{\tau,k}] = 0$, where we define $\cH_{\tau,k}=\sigma(\{Z_{1,k},Y_{1,k},\dots,Z_{\tau-1,k},Y_{\tau-1,k},Z_{\tau,k}\})$ as the history up to time $\tau$. 
By the mean value theorem, 
\$
\sum_{\tau=1}^{t-1} Z_{\tau,k}\epsilon_{\tau,k}
= \nabla L_{t,k}(\theta_k) - \nabla L_{t,k}(\hat\theta_{t,k})
= \nabla^2 L_{t,k}(\tilde\theta_{t,k})(\theta_k - \hat\theta_{t,k})
\$
for some $\tilde\theta_{t,k}$ 
that lies on the segment between $\theta_k$ and $\hat\theta_{t,k}$, 
where $\nabla^2 L_{\tau,k}(\theta)$ 
is the Hessian matrix of $L_{\tau,k}$ at $\theta$. That is, 
\$
\sum_{\tau=1}^{t-1} Z_{\tau,k}\epsilon_{\tau,k}
= \sum_{\tau=1}^{t-1} A''(\tilde\theta_{t,k}^\top Z_{\tau,k})Z_{\tau,k}Z_{\tau,k}^\top (\theta_k - \hat\theta_{t,k}).
\$
Recalling  (ii) $\kappa := \inf_{\|z\|\leq 1, \|\theta-\theta_k\|\leq 1}  {A''}(\theta^\top z) >0$ 
in Assumption~\ref{assump:regularity}(c), 
and noting that $\covmat{t}{k}=\sum_{\tau=1}^{t-1} Z_{\tau,k}Z_{\tau,k}^\top$, 
we know that 
\#\label{eq:bd1}
\Big\|\sum_{\tau=1}^{t-1} Z_{\tau,k}\epsilon_{\tau,k}\Big\|_{\covmat{t}{k}^{-1}}^2 \geq \kappa^2 \|\hat\theta_{t,k}-\theta_k\|_{\covmat{t}{k}}^2
\#
holds as long as $\|\hat\theta_{t,k}-\theta_k\|_2\leq 1$. 

In the sequel, we are to show that 
$\|\hat\theta_{t,k}-\theta_k\|_2\leq 1$ with high probability. 
Let 
\$
S_{t,k}(\theta) = \nabla L_{t,k}(\theta_k) - \nabla L_{t,k}(\theta) 
= \sum_{\tau=1}^{t-1} \big( A'(\theta_k^\top Z_{\tau,k}) - A'(\theta^\top Z_{\tau,k})\big) Z_{\tau,k}.
\$
Note that $A'$ is increasing, hence $A''(\cdot)>0$. Also, as long as $\covmat{t}{k}>0$, we know that $S_{t,k}(\theta)$ is strictly convex in $\theta\in \RR^{d+1}$, and thus $S_{t,k}$ is an injection from $\RR^{d+1}$ 
to $\RR^{d+1}$. 
Furthermore, for any $\theta$ with $\|\theta-\theta_k\|\leq 1$,  
there exists some $\tilde\theta$ that lies on the segment between $\theta$ and $\theta_k$ such that
\$
\|S_{t,k}(\theta) \|_{(\covmat{t}{k})^{-1}}^2 
= \Big\| \sum_{\tau=1}^{t-1} A''(\tilde\theta^\top Z_{\tau,k})  Z_{\tau,k}Z_{\tau,k}^\top (\theta_k -\theta)\Big\|_{(\covmat{t}{k})^{-1}}^2 
\geq \kappa^2 \lambda_{\min}(\covmat{t}{k}) \|\theta_k-\theta\|^2.
\$
The above arguments verify the conditions needed in
\citet[Lemma~A]{chen1999strong}; it then implies  
for any $r>0$, 
\#\label{eq:lemA_chen}
\Big\{\theta\colon \|S_{t,k}(\theta) \|_{(\covmat{t}{k})^{-1}}^2 \leq \kappa ^2 r^2 {\lambda_{\min}(\covmat{t}{k})} \Big\} \subseteq \{\theta\colon \|\theta-\theta_k\|\leq r\}.
\#
Note that Equation~\eqref{eq:lemA_chen} is a deterministic result.

Recall that $S_{t,k}(\hat\theta_{t,k}) = \sum_{\tau=1}^{t-1} Z_{\tau,k}\epsilon_{\tau,k}$, 
and $\epsilon_{\tau,k}$ are mean-zero conditional on $\cH_{\tau,k}$. 
Invoking the concentration inequality of self-normalized processes in Lemma~\ref{lem:self_normalize_t} 
for $\sum_{\tau=1}^{t-1} Z_{\tau,k}\epsilon_{\tau,k}$ 
and $\bar{V}_{t,k} := \lambda I + \sum_{\tau=1}^{t-1} Z_{\tau,k}Z_{\tau,k}^\top$ for some fixed $\lambda>0$, 
with probability at least $1-\delta$, 
it holds for all $t\geq 1$ that 
\$
\Big\|\sum_{\tau=1}^{t-1} Z_{\tau,k}\epsilon_{\tau,k}\Big\|_{\bar{V}_{t,k}^{-1}}^2 \leq 2 \sigma^2 \log\Bigg( \frac{\det (\bar{V}_{t,k})^{1/2} \det(\lambda I)^{-1/2}}{\delta} \Bigg).
\$
Note that $\det(\lambda I) = \lambda^{d+1}$, and $\det(\bar{V}_{t,k})\leq (\lambda_{\max}(\bar{V}_{t,k})\leq (\lambda + t/d)^{d+1}$ since $\|Z_{\tau,k}\|_2\leq 1$ by Lemma~\ref{lem:determ}. We then have 
\#\label{eq:bd_barV}
\Big\|\sum_{\tau=1}^{t-1} Z_{\tau,k}\epsilon_{\tau,k}\Big\|_{\bar{V}_{t,k}^{-1}}^2 \leq 2 \sigma^2 \big( (d+1) \log(1+t/(d\lambda)) + \log(1/\delta)\big) 
\#
with probability at least $1-\delta$ for all $t\geq 1$. Finally, since $\lambda_{\min}(\covmat{t}{k})\geq \lambda_{\min}(\covmat{T_0}{k}) \geq 1$ 
for $t\geq T_0$,  take $\lambda=1/2$ and note that 
\#\label{eq:bd_barV_hatV}
\bar{V}_{t,k} = \lambda I + \covmat{t}{k} 
\preceq (1+\lambda) \covmat{t}{k}
 = 3/2\cdot \covmat{t}{k}.
\#
Combining Equation~\eqref{eq:bd_barV} and~\eqref{eq:bd_barV_hatV},
we have 
\#\label{eq:bd_hatV}
\|S_{t,k}(\hat\theta_{t,k})\|_{(\covmat{t}{k})^{-1}}^2 = 
\Big\|\sum_{\tau=1}^{t-1} Z_{\tau,k}\epsilon_{\tau,k}\Big\|_{(\covmat{t}{k})^{-1}}^2 \leq 3 \sigma^2 \big( (d+1) \log(1+2t/d) + \log(1/\delta)\big) 
\#
for all $t\geq T_0$ with probability at least $1-\delta$. By Lemma~\ref{lem:eigen_Vk}, we know that
$\lambda_{\min}(V_k^{(t)})\geq \frac{3 \sigma^2 }{\kappa^2}\big( (d+1) \log(1+2t/d) + \log(1/\delta)\big)$ holds with probability at least $1-\delta$ since 
$
T_0 \geq \max\big\{  (\frac{32}{3c_1}+\frac{256}{c_1^2} ) \log (\frac{2d+2}{\delta} )  , ~\frac{6\sigma^2}{c_1\kappa^2} ( (d+1) \log(1+2t/d) + \log(1/\delta) )  \big\}.
$ 
Taking a union bound on the above two facts, 
we know that 
\#\label{eq:event_consistent}
{\lambda_{\min}(\covmat{t}{k})} \geq \frac{1}{\kappa^2 }\|S_{t,k}(\theta) \|_{(\covmat{t}{k})^{-1}}^2 
\#
holds
 with probability at least $1-2\delta$.
Then, taking $r=1$ in Equation~\eqref{eq:lemA_chen}, we know that 
$\|\hat\theta_{t,k}-\theta_k\|\leq 1$ on the 
event~\eqref{eq:event_consistent}, which further implies that  $\|\hat\theta_{t,k}-\theta_k\|\leq 1$ for all $t\geq T_0$ with probability at least $1-2\delta$. Applying Equation~\eqref{eq:bd1}, we know that with probability at least $1-2\delta$, 
\$
\|\hat\theta_{t,k}-\theta_k\|_{\covmat{t}{k}}^2 
\leq \frac{1}{\kappa^2}
\Big\|\sum_{\tau=1}^{t-1} Z_{\tau,k}\epsilon_{\tau,k}\Big\|_{\covmat{t}{k}^{-1}}^2 \leq \frac{3 \sigma^2}{\kappa^2} \big( (d+1) \log(1+2t/d) + \log(1/\delta)\big).
\$
Therefore, we conclude the proof of Proposition~\ref{prop:glm_err_theta} by replacing $\delta$ with $\delta/2$.
\end{proof}

\section{Supporting Lemmas}

The following lemma characterizes the deviation of the sample mean of a random matrix. See, e.g., Theorem 1.6.2 of \cite{10.1561/2200000048} and the references therein.

\begin{lemma}[Matrix Bernstein Inequality \citep{10.1561/2200000048}] \label{lem:matrix_concentration}
Suppose that  $\{ A_k\}_{k=1}^n  $ are independent and  centered random matrices in $\RR^{d_1\times d_2}$, that is, $\EE[ A_k] = 0$ for all $k \in [N]$. 
Also, suppose that such random matrices are uniformly upper bounded in the matrix operator norm, that is, $\|A_k\|_{\oper} \leq L$ for all  $k \in [N]  $. Let $Z=\sum_{k=1}^n A_k$ and 
\$
v(Z) = \max\big\{ \|\E[ZZ^\top] \|_{\oper},\|\E[Z^\top Z]\|_{\oper} \big\}= \max\bigg\{ \Big\|\sum_{k=1}^n \E[A_k A_k^\top] \Big\|_{\oper},   \Big\|\sum_{k=1}^n \E[A_k^\top A_k]\Big\|_{\oper}\bigg\}.
\$ 
For all $t\geq 0$, we have 
\begin{equation*}
\PP \big (\|Z\|_{\oper} \geq t \big )\leq (d_1+d_2)\cdot  \exp\Bigl(-\frac{t^2/2}{v(Z)+L / 3 \cdot t }\Bigr).
\end{equation*}
%
\end{lemma}
\begin{proof}
See, e.g., \citet[Theorem 1.6.2]{10.1561/2200000048} for a detailed proof.
\end{proof} 

The following lemma, which is adapted from \cite{abbasi2011improved},  establishes the concentration of 
 self-normalized processes.

\begin{lemma}\label{lem:self_normal}
Let $\{X_t\}_{t=1}^\infty$ be a sequence in $\RR^d$, 
$V\in \RR^{d\times d}$ a positive definite matrix, 
and define $\bar{V}_t = V+\sum_{s=1}^t X_sX_s^\top$. 
If $\|X_t\|_2 \leq L$ for all $t$ , then 
\$
\sum_{t=1}^n \min\big\{1,\|X_t\|_{\bar{V}_{t-1}^{-1}}^2\big\} 
&\leq 2\big( \log \det(\bar{V}_n) - \log\det(V)\big) \\ 
&\leq 2\big( d\log((\tr(V) + nL^2)/d) - \log \det V\big) ,
\$
and finally, if $\lambda_{\min}(V) \geq \max\{1,L^2\}$, then 
 $\sum_{t=1}^n \|X_t\|_{\bar{V}_{t-1}^{-1}}^2 \leq 2 \log \frac{\det (\bar{V}_n)}{\det (V)}$. 
\end{lemma}
\begin{proof}
	See \citet[Lemma~11]{abbasi2011improved} for a detailed proof. 
\end{proof}

\begin{lemma}[Concentration of Self-Normalized Processes \citep{abbasi2011improved}]
\label{lem:self_normalize_t}
    Let $\{\cF_t\}_{t=1}^\infty$ be a filtration. Let $\{\eta_t\}_{t=1}^\infty$ be a real-valued stochastic process such that $\eta_t$ is $\cF_t$-measurable and $\EE[e^{\lambda\eta_t}\given \cF_{t-1}]\leq \exp(\lambda^2 R^2/2)$ for some $R\geq 0$. 
    Let $\{X_t\}_{t=1}^\infty$ be an $\RR^d$-valued stochastic process such that $X_t$ is $\cF_{t-1}$-measurable. Assume that $V$ is a $d\times d$ positive definite matrix. For any $t\geq 0$, define $\bar{V}_t = V+\sum_{s=1}^t X_sX_s^\top$, and $S_t=\sum_{s=1}^t\eta_s X_s$. Then, for any $\delta>0$, with probability at least $1-\delta$, for all $t\geq 0$, 
    \$
\|S_t\|_{\bar{V}_t^{-1}}^2 \leq 2R^2 \log\Bigg( \frac{\det (\bar{V}_t)^{1/2} \det(V)^{-1/2}}{\delta} \Bigg).
    \$
\end{lemma}

\begin{proof}
    See \citet[Theorem 1]{abbasi2011improved}.
\end{proof}

\begin{lemma}[Determinant-Trace Inequality]\label{lem:determ}
Suppose $X_1,X_2,\dots,X_t\in \RR^d$ and for any $1\leq s\leq t$, $\|X_s\|_2\leq L$. 
Let $\bar{V}_t = \lambda I +\sum_{s=1}^t X_sX_s^\top$ for some $\lambda >0$. Then, 
$\det(\bar{V}_t)\leq (\lambda + tL^2/d)^d$.
\end{lemma}

\begin{proof}
    See \citet[Lemma 10]{abbasi2011improved}.
\end{proof}


\end{document}